\documentclass[twoside]{article}

\usepackage[accepted]{aistats2021}

\usepackage[round]{natbib}

\renewcommand{\cite}{\citep}

\usepackage[utf8]{inputenc} %
\usepackage[T1]{fontenc}    %
\usepackage{hyperref}       %
\hypersetup{
    colorlinks=true,
    citecolor=blue,
    linkcolor=blue,
    filecolor=blue,      
    urlcolor=blue,
}
\usepackage{graphicx}
\usepackage{url}            %
\usepackage{booktabs}       %
\usepackage{amsfonts}       %
\usepackage{microtype}      %

\usepackage{calib_paper}

\begin{document}

\twocolumn[

\aistatstitle{
Diagnostic Uncertainty Calibration: Towards Reliable Machine Predictions in Medical Domain
}

\aistatsauthor{ 
  Takahiro Mimori \And 
  Keiko Sasada \And  
  Hirotaka Matsui \And
  Issei Sato
}

\aistatsaddress{ 
  RIKEN AIP
  \And  
  Kumamoto University Hospital%
  \And 
  Kumamoto University
  \And
  The University of Tokyo,
  \\
  RIKEN AIP, ThinkCyte%
} 

]

\begin{abstract}
We propose an evaluation framework for class probability estimates (CPEs)
in the presence of label uncertainty,
which is commonly observed as diagnosis disagreement between experts in the medical domain.
We also formalize evaluation metrics for higher-order statistics, including inter-rater disagreement, to assess 
predictions on label uncertainty.
Moreover, we propose a novel post-hoc method called $\alpha$-calibration, that equips neural network classifiers with calibrated distributions over CPEs.  %
Using synthetic experiments and a large-scale medical imaging application, we show that our approach significantly enhances the reliability of uncertainty estimates: disagreement probabilities and posterior CPEs.
\end{abstract}

\section{Introduction}

The reliability of uncertainty quantification
is %
essential %
for safety-critical systems such as medical diagnosis assistance.
Despite the high accuracy %
of modern neural networks
for a wide range of classification tasks,
their predictive probability often tends to be uncalibrated \cite{guo2017calibration}.
Measuring and improving probability calibration, i.e., %
the consistency of predictive probability for an actual class frequency, has become one of the central issues in machine learning research \cite{vaicenavicius2019evaluating,widmann2019calibration,kumar2019verified}.
At the same time, the uncertainty of ground truth labels in real-world data %
may also affect the reliability of the systems.
Particularly, in the medical domain,
inter-rater variability is commonly observed despite the annotators' expertise
\cite{sasada2018inter,jensen2019improving}.
This variability is also worth predicting for downstream tasks
such as finding examples that need medical second opinions \cite{raghu2018direct}.

To enhance the reliability of class probability estimates (CPEs),
{\it post-hoc} calibration, which transforms output scores %
to fit into empirical class probabilities,
has been proposed %
for both general classifiers \cite{platt1999probabilistic,zadrozny2001obtaining,zadrozny2002transforming} and neural networks \cite{guo2017calibration,kull2019beyond}.
However, current evaluation metrics for calibration rely on empirical accuracy calculated with ground truth, for which the uncertainty of labels has not been considered.
Another problem is that label uncertainty is not fully accounted for by CPEs;
{\it e.g.}, a $50\%$ confidence for class $x$ does not necessarily mean the same amount of human belief, %
even when the CPEs are calibrated.
\citet{raghu2018direct}    %
 indicated that label uncertainty measures, such as an inter-rater disagreement frequency, were biased when they were estimated with CPEs. %
They instead proposed directly discriminating high uncertainty instances with input features.
This treatment, however, requires training an additional predictor for each uncertainty measure %
and lacks an integrated view with the classification task.

In this work, we first develop an evaluation framework for CPEs  %
when label uncertainty is indirectly observed through
multiple annotations per instance (called {\it label histograms}).
Guided with proper scoring rules \cite{gneiting2007strictly} and their decompositions \cite{degroot1983comparison,kull2015novel},
evaluation metrics, including calibration error, are naturally extensible to the situation with label histograms, where we derive estimators that benefit from unbiased or debiased property.
Next, we generalize the framework to evaluate probabilistic predictions on higher-order statistics, including  inter-rater disagreement.
This extension enables us to evaluate these statistics in a unified way with CPEs.
Finally, we address how the reliability of CPEs and disagreement probability estimates (DPEs) can be improved using label histograms.
While the existing post-hoc calibration methods solely address CPEs,
we discuss the importance of obtaining a good predictive distribution over CPEs beyond point estimation to improve DPEs.
Also, the distribution is expected to be useful for obtaining posterior CPEs when expert labels are provided for prediction.
With these insights,
we propose a novel method named $\alpha$-calibration that uses label histograms to equip a neural network classifier with the ability to predict distributions of CPEs.
In our experiments, 
the utility of our evaluation framework and $\alpha$-calibration 
is demonstrated with synthetic data and a large-scale medical image dataset %
with multiple annotations provided from a study of myelodysplastic syndrome (MDS) \cite{sasada2018inter}.
Notably, 
$\alpha$-calibration significantly enhances the quality of DPEs and the posterior CPEs.

In summary, our contributions are threefold as follows:
\begin{itemize}
\item Under ground truth label uncertainty, we develop evaluation metrics that benefit from unbiased or debiased property for class probability estimates (CPEs) using multiple labels per instance, {\it i.e.}, label histograms (Section \ref{sec:cpe_lh}).
\item We generalize our framework to evaluate probability predictions on higher-order statistics, including inter-rater disagreement (Section \ref{sec:higher-order}).
\item We advocate the importance of predicting the distributional uncertainty of CPEs, addressing with a newly devised post-hoc method, $\alpha$-calibration (Section 
\ref{sec:unc_calib_label_hist}).
Our approach substantially improves disagreement probability estimates (DPEs) and posterior CPEs for synthetic and real data experiments (Fig. \ref{fig:mds_disagree1} and Section \ref{sec:experiments}).
\end{itemize}

\begin{figure}[t]
  \centering
  \begin{tabular}{c}
    \whencolumns{
      \includegraphics[width=0.42\columnwidth]{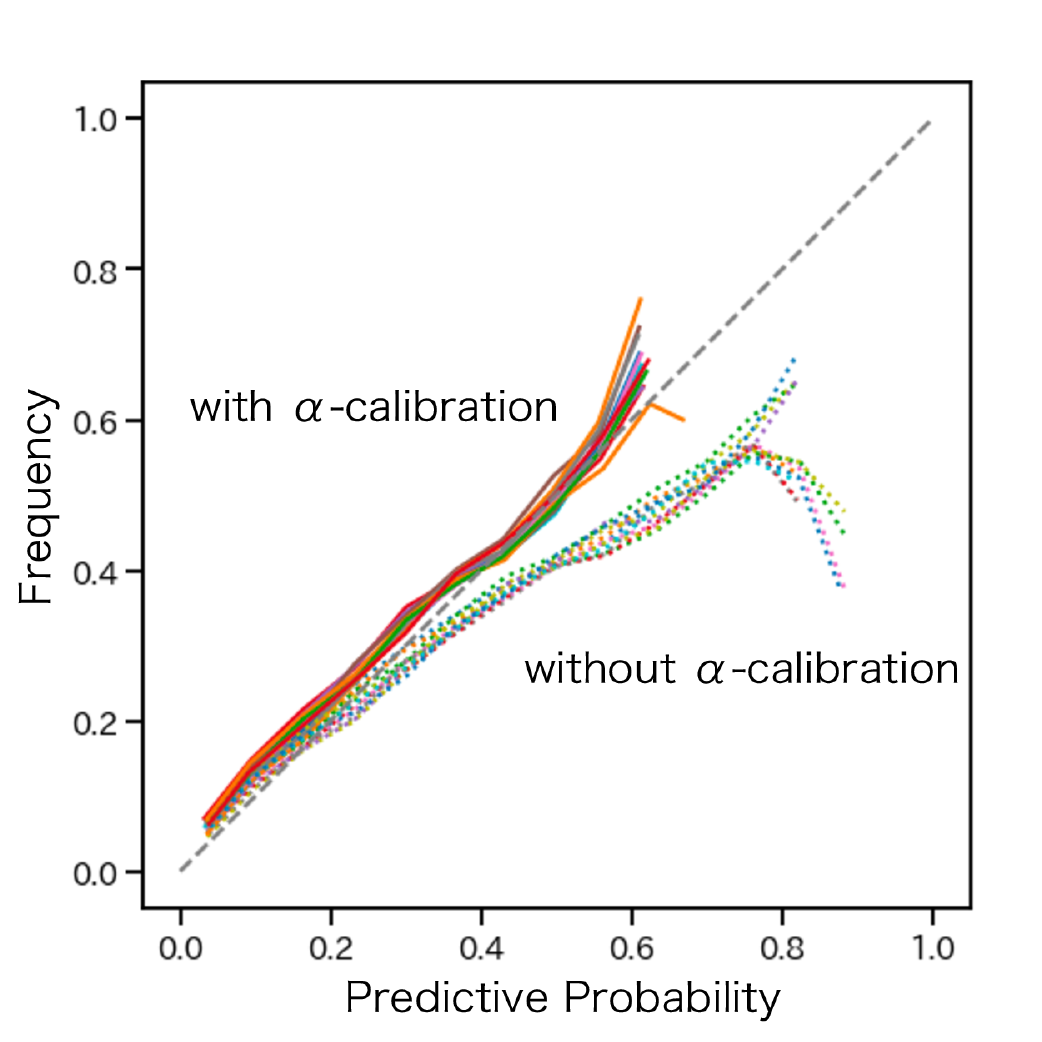}
    }{
      \includegraphics[width=0.78\columnwidth]{figs/mds_disagree_fig1.pdf}
    }
  \end{tabular}
  \caption{
Reliability diagram of disagreement probability estimates (DPEs) in experiments with MDS data: a medical image dataset with multiple labels per instance.
The dashed diagonal line corresponds to perfectly calibrated predictions.
Calibrations of DPEs were significantly enhanced with $\alpha$-calibration (solid lines) from the original ones (dotted lines).}
\label{fig:mds_disagree1}
\end{figure}

\section{Background}
\label{sec:background}
We overview calibration measures, proper scoring rules, and post-hoc calibration of CPEs as a prerequisite for our work.

\paragraph{Notation}
Let $K \in \mathbb{N}$ be a number of categories,
$e^K = \{ e_1, \dots, e_K \}$ be a set of $K$ dimensional one-hot vectors ({\it i.e.}, $e_{kl} := \mathbb{I}[k = l]$),
and $\Delta^{K-1} := \{ \zeta \in \mathbb{R}^K_{\geq 0} : \sum_k \zeta_k = 1 \}$
be a $K-1$-dimensional probability simplex.
Let $(X, Y)$ be jointly distributed random variables over $\mathcal{X}$ and $e^K$,
where $X$ denotes an input feature, such as image data, and $Y$ denotes a $K$-way label.
Let $Z = (Z_1, \dots, Z_K)^\top := f(X) \in \Delta^{K-1}$
denote a random variable that represents class probability estimates (CPEs) for input $X$ with a classifier $f: \mathcal{X} \to \Delta^{K-1}$.

\subsection{Calibration measures}

The notion of calibration, 
which is the agreement between a predictive class probability and an empirical class frequency, 
is important for reliable predictions. %
We reference
\citet{brocker2009reliability,kull2015novel}
for the definition of calibration.
\begin{define}[Calibration]
\label{def:calibration}
\footnote{
  A stronger notion of calibration that requires $Z = C$ is examined in the literature 
  \cite{vaicenavicius2019evaluating,widmann2019calibration}
.}
A probabilistic classifier 
$f: \mathcal{X} \to \Delta^{K-1}$  is said to be {\it calibrated} if 
$Z = f(X)$ matches a true class probability given $Z$,
{\it i.e.},
$\forall k, Z_k = C_k$, where
$C_k := P(Y=e_k | Z)$ and
$C := (C_1, \dots, C_K)^\top \in \Delta^{K-1}$,
which we call a calibration map.
\end{define}
The following metric is commonly used
to measure calibration errors of binary classifiers:
\begin{define}[Calibration error]
\begin{align}
\CE_1 := \pars{\E[|Z_1 - C_1|^p]}^{1/p},
\quad \text{where} \quad
p \geq 1.
\label{eq:CE}
\end{align}
\end{define}
Note that $\CE_1$ takes the minimum value zero iff $Z = C$. 
The cases with $p=1$ and $2$ are called the expectation calibration error (ECE) \cite{naeini2015obtaining} and the squared calibration error \cite{kumar2019verified}, respectively.
Hereafter, we use $p=2$ and let $\CE$ denote $\CE_1$ for binary cases.
For multiclass cases, we denote $\CE$ as a commonly used definition of class-wise calibration error \cite{kumar2019verified}, i.e.,
$(\sum_k \CE_k^2)^{1/2}$.

\subsection{Proper scoring rules} %

Although calibration is a desirable property,
being calibrated is not sufficient for useful predictions.
For instance, 
a predictor that always presents the marginal class frequency $Z = \pars{P(Y=e_1), \dots, P(Y=e_K)}^\top$ is perfectly calibrated, but it entirely lacks the sharpness of prediction for labels stratified with $Z$.
In contrast, the strictly proper scoring rules 
\cite{gneiting2007strictly,parmigiani2009decision}
elicit a predictor's true belief for each instance
and do not suffer from this problem.

\begin{define}[Proper scoring rules for classification]
A loss function $\ell: e^{K} \times \Delta^{K-1} \to \mathbb{R}$
is said to be {\it proper} if
$\forall q \in \Delta^{K-1}$
and 
for all $z \in \Delta^{K-1}$
such that $z \neq q$,
\begin{align}
\E_{Y \sim \mathrm{Cat}(q)}[\ell(Y, z)] \geq \E_{Y \sim \mathrm{Cat}(q)}[\ell(Y, q)]
\end{align}
holds,
where $\mathrm{Cat}(\cdot)$ denotes a categorical distribution.
If the strict inequality holds, $\ell$ is said to be {\it strictly} proper.
Following the convention, we write $\ell(q, z) = \E_{Y \sim \mathrm{Cat}(q)}[\ell(Y, z)]$ for $q \in \Delta^{K-1}$.
\end{define}
For a strictly proper loss $\ell$, the divergence function
$d(q, z) := \ell(q, z) - \ell(q, q)$
takes a non-negative value and is zero iff $z = q$, by definition.
Squared loss $\ell_{\mathrm{sq}}(y, z) := \Verts{y - z}^2$ and
logarithmic loss $\ell_{\log}(y, z) := - \sum_k y_k \log z_k$
are the most well-known examples of strictly proper loss.
For these cases, the divergence functions are given as $d_{\mathrm{sq}}(q, z) = \ell_{\mathrm{sq}}(q, z)$
and $d_{\log}(q, z) = D_{\mathrm{KL}}(q, z)$, a.k.a. KL divergence, respectively.

Let
$L := \E[d(Y, Z)] = \E[d(Y, f(X))]$
denote the expected loss, where the expectation is taken over a distribution $P(X, Y)$.
As special cases of $L$,
\begin{align}
\LBS
&:= \E\left[\ell_{\mathrm{sq}'}(Y, Z)\right]
= \E[(Y_1 - Z_1)^2]
\quad %
(K = 2),
\label{eq:BS}
\\
\LPS
&:= \E[\ell_{\mathrm{sq}}(Y, Z)]
= \E[\Verts{Y - Z}^2]
\quad %
(K \geq 2),
\label{eq:PS}
\end{align}
are commonly used for binary and multiclass prediction, respectively,
where $\ell_{\mathrm{sq}'} := \frac{1}{2} \ell_{\mathrm{sq}}$.
When the expectations are taken over an empirical distribution $\widehat{P}(X, Y)$,
these are referred to as 
Brier score (BS) 
\footnote{
While
\citet{brier1950verification} originally introduced
a multiclass loss that equals $\mathrm{PS}$, we call $\mathrm{BS}$ as Brier score,
following convention
\cite{brocker2012estimating,ferro2012bias}.
}
and probability score (PS),
respectively \cite{brier1950verification,murphy1973new}.

\paragraph{Decomposition of proper losses}
\label{sec:proper_loss_decomp}

The relation between the expected proper loss $L$ and the calibration measures is clarified with a decomposition of $L$ as follows \cite{degroot1983comparison}:
  \begin{align}
  L &= \CL + \RL,
  \quad \text{where}
\whencolumns{ \quad }{ \nonumber \\ & }
  \begin{cases}
  \CL := \E[d(C, Z)],
  & (\text{Calibration Loss})
  \\
  \RL := \E[d(Y, C)].
  & (\text{Refinement Loss})
  \end{cases}
  \label{eq:decomp1}
  \end{align}
The CL term 
corresponds to an error of calibration because the term 
will be zero iff $Z$ equals the calibration map $C = \E[Y|Z]$.
Relations
$\CL_{\hsq} = \CE^2$ and
$\CL_{\sq} = \CE^2$ can be confirmed for binary and multiclass cases, respectively.
Complementarily, the RL term %
shows a dispersion of labels $Y$ given $Z$ from its mean $\E[Y|Z]$ averaged over $Z$.

Under the assumption that labels follow an instance-wise categorical distribution as $Y|X \sim \mathrm{Cat}(Q)$, where $Q(X) \in \Delta^{K-1}$,
\citet{kull2015novel} further decompose $L$ into the following terms:
\begin{align}
L &= \underset{\EL}{\underbrace{\CL + \GL}} + \IL,
\quad \text{where}
\whencolumns{ \quad }{ \nonumber \\ & }
\begin{cases}
\EL = \E[d(Q, Z)], & (\text{Epistemic Loss})
\\
\IL = \E[d(Y, Q)], & (\text{Irreducible Loss})
\\
\GL = \E[d(Q, C)].
& (\text{Grouping Loss})
\end{cases}
\label{eq:decomp2}
\end{align}
The $\mathrm{EL}$ term, which equals zero iff $Z = Q$, is a more direct measure for the optimality of the model than $L$.
The IL term stemming from the randomness of observations is called {\it aleatoric uncertainty} in the literature \cite{der2009aleatory,senge2014reliable}.
We refer to Appendix \ref{aps:psr_decomp} for details and proofs of the statements in this section.

\subsection{Post-hoc calibration for deep neural network classifiers}
\label{sec:bg_calibration}

For deep neural network (DNN) classifiers with the softmax activation,
a post-hoc calibration of class probability estimates (CPEs) is commonly performed by optimizing a linear transformation of the last layer's logit vector 
\cite{guo2017calibration,kull2019beyond},
which minimizes the negative log-likelihood (NLL) of  validation data: 
\begin{align}
\mathrm{NLL} = - \E_{X, Y \sim \widehat{P}
}[ \log P_{\mathrm{obs}}(Y | \widetilde{f}(X)) ],
\label{eq:calib_nll}
\end{align}
where $\widehat{P}$,
$\widetilde{f}: \mathcal{X} \to \Delta^{K-1}$
and $P_{\mathrm{obs}}(Y | Z) = \prod_k Z_k^{Y_k}$
denote
an empirical data distribution,
a transformed DNN function from $f$, 
and a likelihood model, respectively.
More details are described in Appendix \ref{aps:cpe_calib}.
In particular, 
temperature scaling, which has a single parameter and keeps the maximum confidence class unchanged, was the most successful in confidence calibration.
More recent research
\cite{wenger2020non,zhang2020mix,rahimi2020intra}
has proposed nonlinear calibration maps
with favorable properties, such as expressiveness, data-efficiency, and accuracy-preservation.

\section{Evaluation of class probability estimates with label histograms}
\label{sec:cpe_lh}

Now, we formalize evaluation metrics for class probability estimates (CPEs) using label histograms,
where multiple labels per instance are observed.
We assume that $N$ input samples are obtained in an i.i.d. manner: $\{x_i\}_{i=1}^N \sim P(X)$,
and for each instance $i$,
label histogram $y_i \in \mathbb{Z}^K_{\geq 0}$
is obtained from $n_i$ annotators in a conditionally i.i.d. manner, {\it i.e.},
$\{y_i^{(j)} \in e^K \}_{j=1}^{n_i} | x_i \sim P(Y | X=x_i)$
and $y_i = \sum_{j=1}^{n_i} y_i^{(j)}$.
A predictive class probability for the $i$-th instance is denoted by $z_i = f(x_i) \in \Delta^{K-1}$.
In this section, we assume $\ell_{\mathrm{sq}}$ as a proper loss $\ell$
and omit the subscript from terms: $\EL$ and $\CL$ for brevity.
The proofs in this section are found in Appendix \ref{aps:cpe_details}.

\begin{figure*}[t]
  \centering
  \begin{tabular}{cc}
  \whencolumns{
      \includegraphics[width=0.45\linewidth]{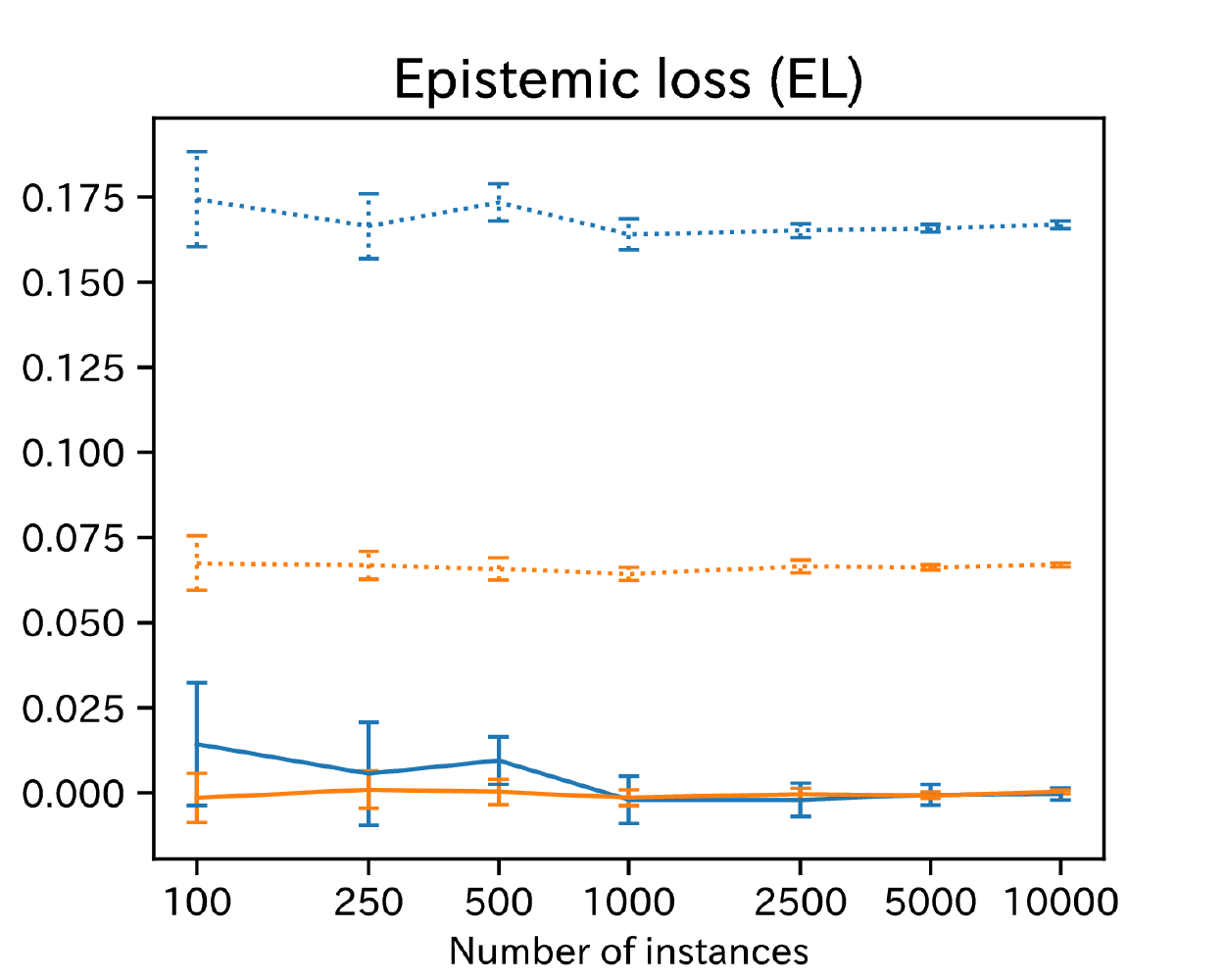}
    & \includegraphics[width=0.45\linewidth]{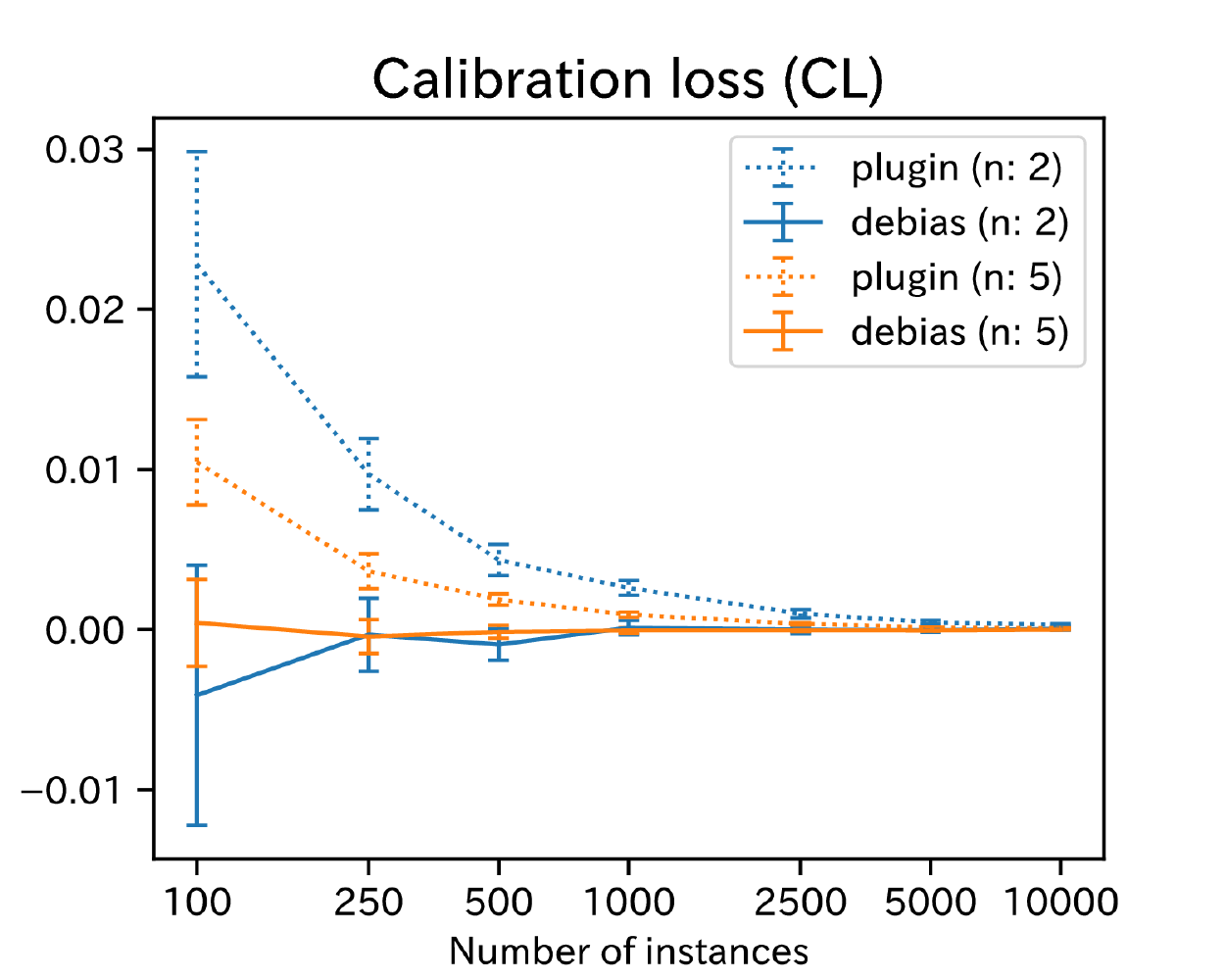}
  }{
      \includegraphics[width=0.4\linewidth]{figs/debias_effect1_EL.pdf}
    & \includegraphics[width=0.4\linewidth]
    {figs/debias_effect1_CL.pdf}
  }
  \end{tabular}
  \caption{Comparison of the plugin and debiased estimators for synthetic data. 
For both $\EL$ and $\CL$, the debiased estimators (solid lines) are closer to ground truth
(which is zero in this experiment) than the plugin estimators (dotted lines).
The error bars show $90\%$ confidence intervals for the means of ten runs.}
\label{fig:el_cl_debias}
\end{figure*}

\subsection{Expected squared and epistemic loss}

We first derive an unbiased estimator of 
the expected squared loss $\LPS$
from label histograms.
\begin{prop}[Unbiased estimator of expected squared loss]
\label{prop:ps_unbias}
The following estimator of $\LPS$ is unbiased.
\begin{align}
\hLPS
&:= \sum_{i=1}^N \frac{w_i}{W}  
\sum_{k=1}^K \left[ (\hmu_{ik} - z_{ik})^2
+ \hmu_{ik} (1 - \hmu_{ik}) \right],
\end{align}
where $\hmu_{ik} := y_{ik}/n_i, \, w_i \geq 0$, and $W := \sum_{i=1}^N w_i$.
\end{prop}
Note that an optimal weight vector $w$ that minimizes the variance $\mathbb{V}[\hLPS]$ would be $w = 1$ if the number of annotators $n_i$ is constant for all instances.
Otherwise, it depends on undetermined terms, as discussed in Appendix \ref{aps:cpe_details}.
We use $w = 1$ as a standard choice,
where $\hL_{\sq}$ coincides with the probability score
 $\PS$ when every instance has a single label.

In addition to letting %
$\hLPS$
have higher statistical power than single-labeled cases,
label histograms also enable us to directly estimate the epistemic loss $\EL$,
which is a discrepancy measure from the optimal model.
A plugin estimator of $\EL$ is obtained as
\begin{align}
\aEL &:= \frac{1}{N} \sum_i \sum_{k} (\hmu_{ik} - z_{ik})^2,
\end{align}
which, however,
turns out to be severely biased.
We alternatively propose the following estimator of $\EL$.
\begin{prop}[Unbiased estimator of $\EL$]
\label{prop:el_unbias}
The following estimator of $\EL$ is unbiased.
\begin{align}
\hEL :=
 \aEL
- \frac{1}{N} \sum_{i} \sum_k \frac{1}{n_i -1} \hmu_{ik} (1 - \hmu_{ik}).
\end{align}
\end{prop}
Note that the second correction term implies that $\hEL$ can only be evaluated when more than one label per instance is available.
The bias correction effect is significant for a small $n_i$,
which is relevant to most of the medical applications.

\subsection{Calibration loss}

Relying on the connection between $\CL$ and $\CE$,
we focus on evaluating $\CL$ to measure calibration.
The calibration loss is further decomposed into class-wise terms as follows:
\begin{align}
&\CL
= \sum_k \CL_k,
\end{align}
where $\CL_k
:= \E[(C_k - Z_k)^2]
= \E[\E[(C_k - Z_k)^2 |Z_k]]$.
Thus, the case of $\CL_{k}$ is sufficient for subsequent discussion.
Note that a difficulty exists in estimating the conditional expectation for $Z_k$.
We take a standard binning-based approach \cite{zadrozny2001obtaining} to evaluate $\CL_k$ by stratifying with $Z_k$ values.
Specifically, $Z_k$ is partitioned into $B_k$ disjoint regions
$\mathcal{B}_k = \{ [\zeta_0=0, \zeta_1), [\zeta_1, \zeta_2), \dots, [\zeta_{B_k-1}, \zeta_{B_k}=1] \}$,
and $\CL_k$ is approximated as follows:
\begin{align}
&\CL_k(\mathcal{B}_k)
:= \sum_{b=1}^{B_k} \CL_{kb}(\mathcal{B}_{k}),
\quad \text{where}
\whencolumns{ \quad }{ \nonumber \\ & }
\begin{cases}
\CL_{kb}(\mathcal{B}_{k})
:= \E[\E[(\bar{C}_{kb} - \bar{Z}_{kb})^2 | Z_k \in \mathcal{B}_{kb}]],
\\
\bar{C}_{kb}
:= \E[Y_k | Z_k \in \mathcal{B}_{kb}],
\\
\bar{Z}_{kb}
:= \E[Z_k | Z_k \in \mathcal{B}_{kb}],
\end{cases}
\end{align}
in which
$\CL_k$ is further decomposed into the bin-wise components.
A plugin estimator of $\CL_{kb}$ is derived as follows:
\begin{align}
&\aCL_{kb}(\mathcal{B}_k) := \frac{\abs{I_{kb}}}{N} (\bar{c}_{kb} - \bar{z}_{kb})^2,
\whencolumns{
  \quad \text{where} \quad
  I_{kb} = \{ i : z_{ik} \in \mathcal{B}_{kb} \},
  \quad
  \bar{c}_{kb} := \frac{\sum_{i \in I_{kb}} \hmu_{ik}}{\abs{I_{kb}}},
  \quad
  \bar{z}_{kb} := \frac{\sum_{i \in I_{kb}} z_{ik}}{\abs{I_{kb}}}.
}{
}
\end{align}
\whencolumns{ 
  Note that $\abs{I_{kb}}$ denotes the size of $I_{kb}$. 
  We can again improve the estimator by debiasing as follows: %
}{
  where, $I_{kb} = \{ i : z_{ik} \in \mathcal{B}_{kb} \}$,
  $\bar{c}_{kb} := \sum_{i \in I_{kb}} \hmu_{ik} / \abs{I_{kb}}$,
  $\bar{z}_{kb} := \sum_{i \in I_{kb}} z_{ik} / \abs{I_{kb}}$,
  and $\abs{I_{kb}}$ denotes the size of $I_{kb}$.
  We can again improve the estimator by debiasing as follows: %
}
\begin{prop}[Debiased estimator of $\CL_{kb}$]
\label{prop:cl_debias}
The plugin estimator of $\CL_{kb}$
is debiased with the following estimator:
\begin{align}
& \hCL_{kb}(\mathcal{B}_k) :=
\aCL_{kb}(\mathcal{B}_k) -
\frac{\abs{I_{kb}}}{N} \frac{\bar{\sigma}_{kb}^2}{\abs{I_{kb}} - 1},
\quad \text{where}
\whencolumns{ \quad }{ \nonumber \\ & \quad }
\bar{\sigma}_{kb}^2 
:= \frac{1}{\abs{I_{kb}}} \sum_{i \in I_{kb}} \hmu_{ik}^2 - \pars{\frac{1}{\abs{I_{kb}}} \sum_{i \in I_{kb}} \hmu_{ik}}^2.
\end{align}
\end{prop}
Note that the correction term against $\aCL_{kb}$ would inflate for small-sized bins with a high 
label variance $\bar{\sigma}_{kb}^2$.
$\hCL_{kb}$ can also be computed for single-labeled data,
{\it i.e.}, $\hmu_{ik} = y_{ik}$.
In this case, the estimator precisely coincides with a debiased estimator for the
{\it reliability} term formerly proposed in meteorological literature \cite{brocker2012estimating,ferro2012bias}.

\subsection{Debiasing effects of  {$\EL$} and  {$\CL$} estimators} %
\label{sec:el_cl_debias}

To confirm the debiasing effect of estimators $\hEL$ and $\hCL$ against the plugin estimators,
we experimented on evaluations of a perfect predictor using synthetic binary labels with varying instance sizes.
For each instance, a positive label probability was drawn from a uniform distribution U(0,1); thereby two or five labels were generated in an {\it i.i.d.} manner.
The predictor indicated the true probabilities so that both EL and CL would be zero in expectation.
As shown in Fig. \ref{fig:el_cl_debias},
the debiased estimators significantly reduced the plugin estimators' biases, even in the cases with two annotators.
Details on the experimental setup are found in Appendix \ref{aps:el_cl_debias_details}.

\section{Evaluation of higher-order statistics}
\label{sec:higher-order}
Here, we generalize our framework to evaluate predictions on higher-order statistics.
As is done for CPEs, the expected proper losses and calibration measures can also be formalized.
We focus on a family of symmetric binary statistics $\phi: e^{K \times n} \to \{0, 1\}$
calculated from $n$ distinct $K$-way labels for the same instance.
For example,
$\phi^{\mathrm{D}} := \mathbb{I}[Y^{(1)} \neq Y^{(2)}]$
represents a disagreement between paired labels $(Y^{(1)}, Y^{(2)})$.
The estimator of $\E[\phi^{\mathrm{D}} | X]$ is known as the Gini-Simpson index,
which is a measure of diversity. %

Given a function $\varphi: \mathcal{X} \to [0, 1]$ that represents a predictive probability of being $\phi = 1$,
the closeness of $\varphi(X)$ to a true probability $P(\phi = 1 | X)$ is consistently evaluated with the expected (one dimensional) squared loss
$L_{\phi} := \E[(\phi - \varphi)^2]$.
Then, the calibration loss $\CL_\phi$ is derived by applying equation \eqref{eq:decomp1} as follows:
\begin{align}
L_{\phi}
= \underset{\CL_{\phi}}{\underbrace{\E[(\E[\phi|\varphi] - \varphi)^2]}}
+ \underset{\RL_{\phi}}{\underbrace{\E[(\phi - \E[\phi|\varphi])^2]}}.
\end{align}
An unbiased estimator of $L_{\phi}$ and 
a debiased estimator of $\CL_\phi$ can be derived following a similar discussion as in CPEs.
The biggest difference from the case of CPEs is that
it requires more careful consideration to obtain an unbiased estimator of $\mu_{\phi,i} := \E[\phi | X = x_i]$ as follows:
\begin{align}
\hmu_{\phi,i}
&:= \binom{n_i}{n}^{-1}
    \sum_{j \in \mathrm{Comb}(n_i, n)} \phi({y}_i^{(j_1)}, \dots, {y}_i^{(j_n)}),
\end{align}
where $\mathrm{Comb}(n_i, n)$ denotes the distinct subset of size $n$ drawn from $\{1, \dots, n_i\}$ without replacement.
The proof directly follows from the fact that $\hmu_{\phi,i}$ is a U-statistic of $n$-sample symmetric kernel function $\phi$ \cite{hoeffding1948class}.
Details on the derivations for $\hL_{\phi}$ and $\hCL_\phi$ are described in Appendix \ref{aps:higher-order}.

\section{Post-hoc uncertainty calibration for DNNs with label histograms}
\label{sec:unc_calib_label_hist}

We consider post-hoc uncertainty calibration problems using label histograms
for a deep neural network (DNN) classifier $f$ that offers CPE with the last layer's softmax activation.

\subsection{Class probability calibration}

For post-hoc calibration of CPEs using label histograms,
existing methods for single-labeled data (Section \ref{sec:bg_calibration})
are straightforwardly extensible
by replacing the likelihood function $P_{\mathrm{obs}}$ in equation \eqref{eq:calib_nll} with a multinomial distribution.

\subsection{Importance of predicting distributional uncertainty of class probability estimates}
\label{sec:cpe_dist_importance}

Although we assume that labels for each input $X$ are sampled from a categorical distribution $Q(X)$ in an i.i.d. manner,
it is important to obtain a reliable CPE distribution beyond point estimation
to perform several application tasks.
We denote such a CPE distribution model as $P(\zeta | X)$,
where $\zeta \in \Delta^{K-1}$. In this case, CPEs are written as $Z = \E[\zeta | X]$.
Below, we illustrate two examples of those tasks.

\paragraph{Disagreement probability estimation}

For each input $X$, the extent of diagnostic disagreement among annotators is 
itself a signal worth predicting,
which is different from classification uncertainty expressed as CPEs.
Specifically, we aim at obtaining a disagreement probability estimation (DPE):
\begin{align}
\varphi^{\mathrm{D}}(X) = \int 1 - \sum_k \zeta_k^2 \, dP(\zeta | X)
\label{eq:dpe}
\end{align}
as a reliable estimator of a probability $P(\phi^{\mathrm{D}}=1 | X)$.
When we only have CPEs, {\it i.e.}, $P(\zeta | X) = \delta(\zeta - f(X))$, 
where $\delta$ denotes the Dirac delta function,
we get $\varphi^{\mathrm{D}} = 1 - \sum_k f(X)_k^2$.
However, 
$\varphi^{\mathrm{D}} \simeq 0$ regardless of $f(X)$
would be more sensible if all the labels are given in unanimous.

\paragraph{Posterior class probability estimates}

We consider a task for updating CPEs of instance $X$ after an expert's annotation $Y$.
Given a CPE distribution model $P(\zeta | X)$, an updated CPEs:
\begin{align}
Z'(X, Y) := \E[\zeta | X, Y],
\label{eq:posterior_cpe}
\end{align}
can be inferred from a Bayesian posterior computation: $P(\zeta | X, Y) \propto P(Y | \zeta) P(\zeta | X)$.
If the prior distribution $P(\zeta|X)$ is reliable,
$Z'$ would be more close to the true value $Q(X)$ than the original CPEs $Z$
in expectation.

\subsection{{$\alpha$}-calibration: post-hoc method for CPE distribution calibration}

\label{sec:alpha_calib}
We propose a novel post-hoc calibration method called {\it $\alpha$-calibration}
that infers a CPE distribution $P(\zeta | X)$ from a DNN classifier $f$ and validation label histograms.
Specifically, we use a Dirichlet distribution
$\mathrm{Dir}(\zeta | \alpha_0(X) f(X))$ to model $P(\zeta | X)$,
and minimize the NLL of label histograms
with respect to instance-wise concentration parameter $\alpha_0(X) > 0$.
We parameterize $\alpha_0$ 
with a DNN that has a shared layer behind the last softmax activation of the DNN $f$
and a successive full connection layer with an $\exp$ activation.
Details are described in Appendix \ref{aps:alpha_calib}.
Using $\alpha_0$ is one of the simplest ways to model the distribution over CPEs; hence it is computationally efficient and less affected by over-fitting without crafted regularization terms.
In addition,
$\alpha$-calibration has several favorable properties:
it is orthogonally applicable with existing CPE calibration methods,
will not degrade CPEs since $Z = \E[\zeta | X] = f(X)$ by design,
and quantities of interest such as a DPE \eqref{eq:dpe} and posterior CPEs \eqref{eq:posterior_cpe} can be computed in closed forms as follows:
\begin{align}
\varphi^{\mathrm{D}} &= \frac{\alpha_0}{\alpha_0 + 1} \pars{1 - \sum_k f_k^2},
\quad
Z' = \frac{\alpha_0 f + Y}{\alpha_0 + 1}.
\label{eq:alpha_dpe_posterior}
\end{align}

\paragraph{Theoretical analysis}

We consider whether a CPE distribution model $P(\zeta | X) = \mathrm{Dir}(\zeta | \alpha_0(X) f(X))$
is useful for downstream tasks.
Let $G = g(X)$ denote a random variable of an output layer
shared between both networks $f$ and $\alpha_0$.
We can write $P(\zeta | X) = P(\zeta | G)$ since
$f$ and $\alpha_0$ are deterministic given $G$.
Although it is unclear whether $P(\zeta | G)$ is an appropriate model for the true label distribution $P(Q | G)$,
we can corroborate the utility of the model with the following analysis.

To evaluate the quality of DPEs and posterior CPEs dependent on $\alpha_0$,
we analyze the expected loss 
$L_{\phi^{\mathrm{D}}} = \E_G[L_{\phi^{\mathrm{D}}, G}]$
and the epistemic loss
$\EL' := \E_G[\EL'_G]$, respectively,
where we define $L_{\phi^{\mathrm{D}}, G} := \E[(\phi^{\mathrm{D}} - \varphi^{\mathrm{D}})^2 | G]$
and $\EL_{G}' := \E[\sum_k (Z'_k - Q_k)^2 | G]$.
We denote those for the original model $P_0(\zeta | X) = \delta(\zeta - f(X))$
before $\alpha$-calibration as $L_{\phi^{\mathrm{D}},G}^{(0)}$ and $\EL_{G}'^{(0)}$, respectively.
\begin{theorem}
\label{th:alpha_calib_analysis}
There exist intervals for parameter $\alpha_0 \geq 0$, %
which improve task performances as follows.
\begin{enumerate}
\item For DPEs, $L_{\phi^{\mathrm{D}},G} \leq L_{\phi^{\mathrm{D}},G}^{(0)}$ holds when $(1 - 2u_Q + s_Z)/2(u_Q - s_Z) \leq \alpha_0$,
and $L_{\phi^{\mathrm{D}},G}$ takes the minimum value
when $\alpha_0 = (1 - u_Q)/(u_Q - s_Z)$,
if $u_Q > s_Z$ is satisfied.
\item For posterior CPEs, $\EL_G' \leq \EL_G'^{(0)}$ holds when $(1 - u_Q - \EL_G) / 2 \EL_G \leq \alpha_0$,
and $\EL_G'$ takes the minimum value when $\alpha_0 = (1 - u_Q) / \EL_G$,
if $\EL_G > 0$ is satisfied.
\end{enumerate}
Note that we denote $s_Z := \sum_k Z_k^2$, $u_Q := \E[\sum_k Q_k^2 | G]$, $v_Q := \V[Q_k | G]$,
and $\EL_G := \E[\sum_k (Z_k - Q_k)^2 | G]$.
The optimal $\alpha_0$ of both tasks coincide to be $\alpha_0 = (1 - u_Q) / v_Q$,
if CPEs match the true conditional class probabilities given $G$, {\it i.e.}, $Z = \E[Q | G]$.
\end{theorem}
The proof is shown in Appendix \ref{aps:alpha_calib_analysis_proof}.

\section{Related work}

\paragraph{Noisy labels}
Learning classifiers under label uncertainty has also been studied 
as a noisy label setting,
assuming unobserved ground truth labels and label noises.
The cases of uniform or class dependent noises
have been studied to ensure robust learning schemes \cite{natarajan2013learning,jiang2018mentornet,han2018co}
and predict potentially inconsistent labels
\cite{northcutt2019confident}. 
Also, there have been algorithms that modeled a generative process of noises depending on input features
\cite{xiao2015learning,liu2020timely}. %
However, the paradigm of noisy labels requires qualified gold standard labels to validate 
predictions,
while we assume that ground truth labels include uncertainty.

\paragraph{Multiple annotations}
Learning from multiple annotations per instance has also been studied in crowdsourcing field \cite{guan2017said,rodrigues2017deep,tanno2019learning},
which particularly modeled labelers with heterogeneous skills, occasionally including non-experts. %
In contrast, we focus on instance-wise uncertainty
under homogeneous expertise %
as in \citet{raghu2018direct}.
Another related paradigm is label distribution learning
\cite{geng2016label,gao2017deep}, which assumes instance-wise categorical probability as ground truth. Whereas they regard the true probability as observable, we assume it as a hidden variable on which actual labels depend.

\paragraph{Uncertainty of CPEs}
Approaches for predicting distributional uncertainty of CPEs for DNNs have mainly studied as part of Bayesian modeling.
\citet{
gal2016dropout,
lakshminarayanan2017simple,
teye2018bayesian,
wang2019aleatoric}
found practical connections for using ensembled DNN predictions as approximate Bayesian inference
and uncertainty quantification   \cite{kendall2017uncertainties},
which however require additional computational cost for sampling.
An alternative approach is directly modeling CPE distribution with parametric families.
In particular,
\citet{
sensoy2018evidential,
malinin2018predictive,
sadowski2018neural,
joo2020being}
adopted the Dirichlet distribution for a tractable distribution model
and used for applications, such as detecting out-of-distribution examples.
However, the use of multiple labels have not been explored in these studies. Moreover, these approaches need customized training procedures from scratch and are not designed to apply for DNN classifiers in a post-hoc manner, as is done in $\alpha$-calibration.

\section{Experiments}
\label{sec:experiments}
We applied DNN classifiers and calibration methods to synthetic and real-world image data with label histograms, where the performance was evaluated with our proposed metrics.
Especially, we demonstrate the utility of $\alpha$-calibration in two applications: predictions on inter-rater label disagreement (DPEs) and posterior CPEs,
which we introduced in Section \ref{sec:cpe_dist_importance}.
Our implementation is available online
\footnote{\url{https://github.com/m1m0r1/lh_calib}}.

\subsection{Experimental setup}
\label{sec:experiment_setup}

\paragraph{Synthetic data}
We generated two synthetic image dataset: Mix-MNIST and Mix-CIFAR-10 from MNIST \cite{lecun2010mnist} and CIFAR-10 \cite{krizhevsky2009learning}, respectively.
We randomly selected half of the images to create mixed-up images from pairs and the other half were retained as original.
For each of the paired images, a random ratio that followed a uniform distribution $U(0, 1)$ was used for the mix-up and a class probability of multiple labels,
which were two or five in validation set.

\paragraph{MDS data}

We used a large-scale medical imaging dataset for myelodysplastic syndrome (MDS) \cite{sasada2018inter}, which contained over $90$ thousand hematopoietic cell images obtained from blood specimens from $499$ patients with MDS.
{
This study was carried out in collaboration with medical technologists
who mainly belonged to the Kyushu regional department of the Japanese Society for Laboratory Hematology. 
The use of peripheral blood smear samples for this study was approved by the ethics committee at Kumamoto University, and the study was performed in accordance with the Declaration of Helsinki.
}
For each of the cellular images, a mean of $5.67$ medical technologists annotated the cellular category from $22$ subtypes,
where accurate classification according to the current standard criterion was still challenging for technologists with expertise.

\paragraph{Compared methods}
We used DNN classifiers as base predictors (Raw) for CPEs,
where a three layered CNN architecture for Mix-MNIST and a VGG16-based one for Mixed-CIFAR-10 and MDS were used.
For CPE calibration, we adopted temperature scaling (ts), which was widely used for DNNs \cite{guo2017calibration}.
To predict CPE distributions, we used $\alpha$-calibration  and ensemble-based methods: Monte-Calro dropout (MCDO) \cite{gal2016dropout} and test-time augmentation (TTA) \cite{ayhan2018test}, which were both applicable to DNNs at prediction-time. Note that TTA was only applied for Mix-CIFAR-10 and MDS, in which we used data augmentation while training.
We also combined ts and/or $\alpha$-calibration with the ensemble-based methods in our experiments,
while some of their properties, including the invariance of accuracy for ts and that of CPEs for $\alpha$-calibration, were not retained for these combinations.
The details of the network architectures and parameters were described in Appendix \ref{aps:results_cpes}.
Considering a constraint of the high labeling costs with experts in the medical domain,
we focused on scenarios that
training instances were singly labeled 
and multiple labels were only available for the validation and test set.

\subsection{Results}

\begin{table*}[!h]
\caption{Evaluations of disagreement probability estimates (DPEs)}
\vspace{.2cm}
\label{tb:eval_DPEs}
\centering
\begin{tabular}{lcccccccccc}
\toprule
&
\multicolumn{2}{c}{Mix-MNIST(2)}
&
\multicolumn{2}{c}{Mix-MNIST(5)}
&
\multicolumn{2}{c}{Mix-CIFAR-10(2)}
&
\multicolumn{2}{c}{Mix-CIFAR-10(5)}
&
\multicolumn{2}{c}{MDS}
\\
\cmidrule(lr){2-3}
\cmidrule(lr){4-5}
\cmidrule(lr){6-7}
\cmidrule(lr){8-9}
\cmidrule(lr){10-11}
Method 
 & $\hL_{\phi^D}$ & $\hCE_{\phi^D}$ %
 & $\hL_{\phi^D}$ & $\hCE_{\phi^D}$ %
 & $\hL_{\phi^D}$ & $\hCE_{\phi^D}$ %
 & $\hL_{\phi^D}$ & $\hCE_{\phi^D}$ %
 & $\hL_{\phi^D}$ & $\hCE_{\phi^D}$ %
\\
\midrule
Raw & .0755 & .0782     & .0755 & .0782 & .1521 & .2541 & .1521 & .2541 & .1477 & .0628
\\
Raw+$\alpha$ & {\bf .0724} & {\bf .0524}        & {\bf .0724} & {\bf .0531}     & {\bf .0880} & {\bf .0357}
        & {\bf .0877} & {\bf .0322}     & {\bf .1454} & {\bf .0406}
\\
\cmidrule[.05mm]{1-11}
Raw+ts & .0775 & .0933  & .0773 & .0923 & .1968 & .3310 & .1978 & .3324 & .1482 & .0663
\\
Raw+ts+$\alpha$ & {\bf .0699} & {\bf .0344}     & {\bf .0702} & {\bf .0379}     & {\bf .0863} & {\bf .0208}
        & {\bf .0861} & {\bf \underline{.0164}} & {\bf .1445} & {\bf .0261}
\\
\midrule
MCDO & .0749 & .0728    & .0749 & .0728 & .1518 & .2539 & .1518 & .2539 & .1470 & .0562
\\
MCDO+$\alpha$ & {\bf .0700} & {\bf .0277}       & {\bf .0700} & {\bf .0285}     & {\bf .0873} & {\bf .0275}
        & {\bf .0870} & {\bf .0241}     & {\bf .1450} & {\bf .0346}
\\
\cmidrule[.05mm]{1-11}
MCDO+ts & .0805 & .1062 & .0802 & .1049 & .1996 & .3353 & .2002 & .3362 & .1479 & .0635
\\
MCDO+ts+$\alpha$ & {\bf \underline{.0690}} & {\bf \underline{.0155}}    & {\bf \underline{.0691}} & {\bf \underline{.0188}}        & {\bf .0863} & {\bf \underline{.0196}} & {\bf .0861} & {\bf .0167}     & {\bf .1442} & {\bf \underline{.0186}}
\\
\midrule
TTA & NA & NA & NA & NA & .1677 & .2856 & .1677 & .2856 & .1441 & .0488
\\
TTA+$\alpha$ & NA & NA & NA & NA & {\bf \underline{.0860}} & {\bf .0245} & {\bf \underline{.0857}} & {\bf .0231} & {\bf .1428} & {\bf .0334}
\\
\cmidrule[.05mm]{1-11}
TTA+ts & NA & NA & NA & NA & .2421 & .3957 & .2430 & .3968 & .1448 & .0553
\\
TTA+ts+$\alpha$ & NA & NA & NA & NA & {\bf .0872} & {\bf .0467}     & {\bf .0870} & {\bf .0398}     & {\bf \underline{.1422}} & {\bf .0197}
\\
\bottomrule
\end{tabular}
\end{table*}

\begin{table*}[!h]
\caption{Epistemic losses ($\hEL$) 
of prior and posterior class probability estimates (CPEs)}
\vspace{.2cm}
\label{tb:posterior_CPEs}
\centering
\begin{tabular}{lcccccccccc}
\toprule
&
\multicolumn{2}{c}{Mix-MNIST(2)}
&
\multicolumn{2}{c}{Mix-MNIST(5)}
&
\multicolumn{2}{c}{Mix-CIFAR-10(2)}
&
\multicolumn{2}{c}{Mix-CIFAR-10(5)}
&
\multicolumn{2}{c}{MDS}
\\
\cmidrule(lr){2-3}
\cmidrule(lr){4-5}
\cmidrule(lr){6-7}
\cmidrule(lr){8-9}
\cmidrule(lr){10-11}
Method 
 & Prior & Post. %
 & Prior & Post. %
 & Prior & Post. %
 & Prior & Post. %
 & Prior & Post. %
\\
\midrule
Raw+$\alpha$ & .0388 & {\bf \underline{.0292}}  & .0388 & {\bf \underline{.0292}}       & .2504 & {\bf .0709}      & .2504 & {\bf .0693}   & .0435 & {\bf .0354}
\\
Raw+ts+$\alpha$ & \underline{.0379} & {\bf .0293}       & \underline{.0379} & {\bf .0298}       & .2423 & {\bf \underline{.0682}}  & .2423 & {\bf \underline{.0676}}       & .0430 & {\bf \underline{.0352}}
\\
\cmidrule[.05mm]{1-11}
MCDO & .0395 & {\bf .0391}      & .0395 & {\bf .0391}   & .2473 & {\bf .2471}   & .2473 & {\bf .2471}   & {\bf .0437} & .0440
\\
MCDO+ts & .0410 & {\bf .0406}   & .0404 & {\bf .0400}   & .2428 & {\bf .2425}   & .2431 & {\bf .2428}   & {\bf .0435} & .0438
\\
TTA & NA & NA   & NA & NA       & \underline{.2216} & {\bf .2184}       & \underline{.2216} & {\bf .2184} & {\bf \underline{.0378}} & .0382
\\
TTA+ts & NA & NA        & NA & NA       & .2452 & {\bf .2428}   & .2451 & {\bf .2427}   & {\bf .0379} & .0383
\\
\bottomrule
\end{tabular}
\end{table*}

\paragraph{Class probability estimates}
We observed a superior performance of TTA in accuracy and $\hEL$ and a consistent improvement in $\hEL$ and $\hCL$ with ts,
for all the dataset.
The details are found in Appendix \ref{aps:results_cpes}.
By using $\hEL$, 
the relative performance of CPE predictions had been clearer than $\hL$ since the irreducible loss was subtracted from $\hL$.
We include additional MDS experiments using full labels in Appendix \ref{aps:mds_full_results}, which show similar tendencies but improved overall performance.

\paragraph{Disagreement probability estimates}
We compared squared loss and calibration error of DPEs for combinations of prediction schemes (Table \ref{tb:eval_DPEs}
\footnote{The mechanisms that cause the degradation of DPEs for Raw+ts are discussed in Appendix \ref{aps:discussion_ts_for_dpe}.}).
Notably, $\alpha$-calibration combined with any methods showed a consistent and significant  decrease %
in both $\hL_{\phi^{\mathrm{D}}}$ and $\hCE_{\phi^{\mathrm{D}}}$,
in contrast to MCDO and TTA, which had not solely improved the metrics.
The improved calibration was also visually confirmed with a reliability diagram of DPEs for MDS data (Fig. \ref{fig:mds_disagree1}).

\paragraph{Posterior CPEs}
We evaluated posterior CPEs, 
when one expert label per instance was available for test set.
This task required a reasonable prior CPE model to update belief with additional label information.
We summarize $\hEL$ metrics of prior and posterior CPEs 
for combinations of dataset and prediction methods in Table \ref{tb:posterior_CPEs}.
As we expected, $\alpha$-calibration significantly decreased losses of the posterior CPEs, i.e., they got closer to the ideal CPEs than the prior CPEs.
While TTA showed superior performance for the prior CPEs,
the utility of the ensemble-based methods for the posterior computation was limited.
We omit experiments on MCDO and TTA combined with $\alpha$-calibration,
as they require further approximation to compute posteriors.

\section{Conclusion}
In this work, we have developed a framework for evaluating probabilistic classifiers under ground truth label uncertainty, accompanied with useful metrics that benefited from unbiased or debiased properties.
The framework was also generalized to evaluate higher-order statistics, including inter-rater disagreements.
As a reliable distribution over class probability estimates (CPEs) is essential for higher-order prediction tasks, 
such as disagreement probability estimates (DPEs) and posterior CPEs,
we have devised a post-hoc calibration method called $\alpha$-calibration,
which directly used multiple annotations to improve CPE distributions.
Throughout empirical experiments with synthetic and  real-world medical image data,
we have demonstrated the utility of the evaluation metrics in performance comparisons
and a substantial improvement in DPEs
and posterior CPEs with $\alpha$-calibration.

\section*{Acknowledgements}
IS was supported by JSPS KAKENHI Grant Number 20H04239 Japan. 
This work was supported by RAIDEN computing system at RIKEN AIP center.

\bibliographystyle{plainnat}      %
\bibliography{calib_paper}      %

\appendix
\onecolumn

\aistatstitle{
	Diagnostic Uncertainty Calibration: Towards Reliable Machine Predictions in Medical Domain (Appendix)
}

\section{Background for proper loss decomposition}
\label{aps:psr_decomp}
We describe the proofs for proper loss decompositions introduced in section
\ref{sec:background}.
\subsection{Decomposition of proper losses and calibration}

As we have described in section \ref{sec:proper_loss_decomp},
the expected loss $L$ can be decomposed as follows:
\begin{theorem}[\citet{degroot1983comparison}]
\label{th:decomp_cl_rl}
The expectation of proper loss $\ell$ is decomposed into non-negative terms as follows:
\begin{align}
L = \CL + \RL,
\quad \text{where} \quad
\begin{cases}
\CL := \E[d(C, Z)],
& (\text{Calibration Loss})
\\
\RL := \E[d(Y, C)],
& (\text{Refinement Loss})
\end{cases}
\end{align}
where a calibration map $C := \E[Y | Z] \in \Delta^{K-1}$ is defined as in Def. \ref{def:calibration}.
\end{theorem}
\begin{proof}
\begin{align*}
L &= \E[d(Y, Z)]
\\
&= \E[\ell(Y, Z) - \ell(Y, Y)]
\\
&= 
 \E[\ell(Y, Z) - \ell(Y, C)]
+ \E[\ell(Y, C) - \ell(Y, Y)]
\\
&= 
\E[\E[\ell(Y, Z) - \ell(Y, C)|Z]]
+ \E[d(Y, C)],
\end{align*}
where the second term equals to the RL term.
For the first term,
as we have defined $\ell(q, Z) := \E_{Y \sim \mathrm{Cat}(q)}[\ell(Y, Z)]$
when $q \in \Delta^{K-1}$, the subterms can be rewritten as follows:
\begin{align*}
\E[\ell(Y, Z) | Z]
&= \E_{Y \sim \mathrm{Cat}(\E[Y|Z])}[\ell(Y, Z)]
= \ell(\E[Y | Z], Z) = \ell(C, Z),
\\
\E[\ell(Y, C) | Z]
&= \E_{Y \sim \mathrm{Cat}(\E[Y|Z])}[\ell(Y, C)]
= \ell(\E[Y|Z], C)
= \ell(C, C).
\end{align*}
Hence, the first term equals to the CL term as follows:
\begin{align*}
\E[\E[\ell(Y, Z) - \ell(Y, C)|Z]]
&=
\E[\E[\ell(C, Z) - \ell(C, C)|Z]]
=
\E[d(C, Z)].
\end{align*}
\end{proof}
Note that we have followed the terminology used in \citet{kull2015novel}. The terms $\CL$ and $\RL$ are also referred to as reliability \cite{brocker2012estimating,ferro2012bias} and sharpness \cite{degroot1983comparison}.

\subsection{Decomposition of proper losses under label uncertainty}

As we have described in section \ref{sec:proper_loss_decomp},
if $Y$ follows an instance-wise categorical distribution with a probability vector, {\it i.e.}, $Y|X \sim \mathrm{Cat}(Q), \,\text{where}\, Q(X) \in \Delta^{K-1}$,
$L$ can be further decomposed as follows:
\begin{theorem}[\citet{kull2015novel}]
\label{th:kull_decomp}
The expectation of proper loss $\ell$ is decomposed into non-negative terms as follows:
\begin{align}
L &= \EL+ \IL
= \underset{\EL}{\underbrace{\CL + \GL}} + \IL,
\\
& \quad \text{where} \quad
\begin{cases}
\EL = \E[d(Q, Z)], & (\text{Epistemic Loss})
\\
\IL = \E[d(Y, Q)], & (\text{Irreducible Loss})
\\
\GL = \E[d(Q, C)].
& (\text{Grouping Loss})
\end{cases}
\end{align}
Note that the CL term is the same form as in equation \eqref{eq:decomp1}.
\end{theorem}
\begin{proof}
We first prove the first equality.
\begin{align*}
L
&= \E[d(Y, Z)]
\\
&= \E[\ell(Y, Z) - \ell(Y, Y)]
\\
&=
 \E[\ell(Y, Z) - \ell(Y, Q)]
+ \E[\ell(Y, Q) - \ell(Y, Y)]
\\
&=
 \E[\E[\ell(Y, Z) - \ell(Y, Q) | Q]]
+ \E[d(Y, Q)],
\end{align*}
where the second term is $\IL$.
As similar to the proof of Theorem \ref{th:decomp_cl_rl}, the following relations hold:
\begin{align*}
\E[\ell(Y, Z) | Q] &= \E_{Y \sim Q}[\ell(Y, Z) | Q] = \E[\ell(Q, Z) | Q],
\\
\E[\ell(Y, Q) | Q] &= \E_{Y \sim Q}[\ell(Y, Q) | Q] = \E[\ell(Q, Q) | Q].
\end{align*}
Therefore, the first term turns out to be $\EL$ as follows:
\begin{align*}
\E[\E[\ell(Y, Z) - \ell(Y, Q) | Q]]
= \E[\E[\ell(Q, Z) - \ell(Q, Q) | Q]]
= \E[d(Q, Z)].
\end{align*}
This term is further decomposed as follows:
\begin{align*}
\E[d(Q, Z)]
&= \E[\ell(Q, Z) - \ell(Q, C)] + \E[\ell(Q, C) - \ell(Q, Q)]
\\
&= \E[\E[\ell(Q, Z) - \ell(Q, C) | Z]] + \E[d(Q, C)],
\end{align*}
where the second term is $\GL$.
To show that the first term is $\CL$, we have to prove the following results:
\begin{align*}
\E[\ell(Q, Z) | Z] &= \E[\ell(C, Z) | Z],
\\
\E[\ell(Q, C) | Z] &= \E[\ell(C, C) | Z].
\end{align*}
As these are proven with the same procedure,
we only show the proof for the first equality.
\begin{align*}
\E[\ell(Q, Z) | Z]
&= \E[\E_{Y \sim Q} \ell(Y, Z) | Z]
\\
&= \E[\sum_k \ell(Y_k, Z) Q_k | Z]
\\
&= \E[\sum_k \ell(Y_k, Z) \E[Y|Z]_k | Z]
\\
&= \E[\sum_k \ell(Y_k, Z) C_k | Z]
\\
&= \E[\E_{Y \sim C}[\ell(Y, Z)] | Z]
= \E[\ell(C, Z) | Z].
\end{align*}
\end{proof}
Theorems and proofs for more generalized decompositions are found in \citet{kull2015novel}.

\section{Details on CPE evaluation metrics with label histograms}
\label{aps:cpe_details}

We describe
supplementary information for Section \ref{sec:cpe_lh}:
proofs for propositions, additional discussion, and experimental setup.

\subsection{Unbiased Estimators of {$\LPS$}}
We give a proof for Prop. \ref{prop:ps_unbias}.
\begin{repprop}{prop:ps_unbias} [Unbiased estimator of expected squared loss]
\label{prop:ps_unbias2}
The following estimator of $\LPS$ is unbiased.
\begin{align}
\hLPS
&:= \frac{1}{W} \sum_{i=1}^N w_i 
\sum_{k=1}^K \left[ (\hmu_{ik} - z_{ik})^2
+ \hmu_{ik} (1 - \hmu_{ik}) \right],
\end{align}
where $\hmu_{ik} := y_{ik}/n_i, \, w_i \geq 0$, and $W := \sum_{i=1}^N w_i$.
\end{repprop}
\begin{proof}
We begin with the following plugin estimator of $\LPS$ with an instance $i$:
\begin{align}
\aLPS[{i}] := \sum_k \frac{1}{n_i} \sum_{j=1}^{n_i} (y_{ik}^{(j)} - z_{ik})^2.
\end{align}
By taking an expectation with respect to $y_{ik}$ and $z_{ik}$,
\begin{align*}
\E[\aLPS[{i}]]
&= \frac{1}{n_i} \sum_{j=1}^{n_i} \sum_k \E[(y_{ik}^{(j)} - z_{ik})^2]
= \sum_k \E[(Y_k - Z_k)^2] = \E[\Verts{Y - Z}^2] = \LPS.
\end{align*}
Therefore, $\aLPS[i]$ is an unbiased estimator of $\LPS$.
Intuitively, an estimator combined with $N$ instances is expected to have a lower variance than that with a single instance.
A linear combination of $\aLPS[1], \dots, \aLPS[N]$ is also an unbiased estimator as follows:
\begin{align*}
\frac{1}{W} \E[\sum_i w_i \aLPS[i]] = \frac{1}{W} \sum_i w_i \E[\aLPS[i]] = \LPS,
\end{align*}
where $\sum_i w_i \geq 0, W := \sum_i w_i$.
The proof completes by transforming $\aLPS[i]$ as follows:
\begin{align*}
\aLPS[{i}] 
&= \sum_k \frac{1}{n_i} \sum_{j=1}^{n_i} \pars{y_{ik}^{(j)} - 2y_{ik}^{(j)} z_{ik} + z_{ik}^2}
= \sum_k
\pars{\hmu_{ik} - 2\hmu_{ik} z_{ik} + z_{ik}^2}
= \sum_k 
\hmu_{ik} (1 - \hmu_{ik}) + \pars{\hmu_{ik} - z_{ik}}^2.
\end{align*}
\end{proof}

\paragraph{Determination of weights $\bm{w}$}
For the undetermined weights $w_1, \dots, w_N$, we have argued that the optimal weights would be constant when the numbers of annotators $n_1, \dots, n_N$ were constant.
As we assume that each of an instance $i$ follows an independent categorical distribution with a parameter $Q_i \in \Delta^{K-1}$,
the variance of $\hLPS[k]$ is decomposed as follows:
\begin{align}
\V[\hLPS[k]] = \sum_{i=1}^N \pars{\frac{w_i}{W}}^2 \V[\hLPS[{ik}]].
\end{align}
Thus, if $n_i$ is constant for all the instance, the optimal weights are found as follows:
\begin{align}
\min_{\bm{w}'} \sum_{i=1}^N w_i'^{2}
\quad s.t. \quad
\sum_{i=1}^N w_i' = 1,\quad \forall i, \, w_i' \geq 0,
\end{align}
By taking a derivative with respect to $\bm{w}'$ 
of $\sum_{i=1} w_i'^{2} + \lambda (\sum_{i=1}^N w_i' - 1)$,
the solution is $\forall i,\, w'_i = 1/N$.

For cases with varying numbers of annotators per instance,
it is not straightforward to determine the optimal weights.
From a standard result of variance formulas, the variance of $\hLPS$ is further decomposed as follows:
\begin{align*}
\V[\hLPS[{ik}]] 
&:= \E[\V[\hLPS[{ik}] | X_i]] + \V[\E[\hLPS[{ik}] | X_i]]
\\
&= \frac{1}{n_i} \E[\sigma_{\sq,k}^2(X)] + \V[\mu_{\sq,k}(X)],
\end{align*}
where $\mu_{\sq,k}(X) := \E[(Y_k - Z_k)^2 | X]$ 
and $\sigma^2_{\sq,k}(X) := \V[(Y_k - Z_k)^2 | X]$.
Therefore, the optimal weights depend on the ratio of the first and the second terms.
If the first term is negligible compared to the second term,
using the constant weights regardless of $n_i$ would be optimal.
In contrast, if the first term is dominant,
$w_i \propto n_i$ would be optimal.
However, the ratio of the two terms depend on the dataset and is not determined a priori.
In this work, we have used $w_i = 1$.

\subsection{Unbiased Estimators of EL}

In this section, we give a proof for Prop. \ref{prop:el_unbias}.
\begin{define}[Plugin estimator of EL]
\begin{align}
\aEL := \frac{1}{N} \sum_{i=1}^N \sum_{k=1}^{K} (\hmu_{ik} - z_{ik})^2.
\end{align}
\end{define}
\begin{repprop}{prop:el_unbias}[Unbiased estimator of EL]
\label{prop:el_unbias2}
The following estimator of EL is unbiased.
\begin{align}
\hEL :=
 \aEL
- \frac{1}{N} \sum_{i} \sum_k \frac{1}{n_i -1} \hmu_{ik} (1 - \hmu_{ik}).
\end{align}
\end{repprop}
\begin{proof}
The term EL is decomposed as $\EL = \sum_k \EL_k$, where
\begin{align*}
\EL_k = \E[(Q_k - Z_k)^2] = 
\E[Q_k^2] - 2 \E[Q_k Z_k] + \E[Z_k^2].
\end{align*}
As for the terms in the plugin estimator $\aEL = \sum_k \aEL_k$, we can show that
\begin{align*}
\E[\hmu_{ik} z_{ik}] &= \frac{1}{n_i} n_i \E[Q_k Z_k] = \E[Q_k Z_k],
\\
\E[z_{ik}^2] &= \E[Z_k^2].
\end{align*}
The bias of $\aEL_k$ comes from $\hmu_{ik}^2$,
that corresponds to $\E[Q_k^2]$ term.
We can replace $\hmu_{ik}^2$ by an unbiased estimator as follows:
\begin{align}
\hmu_{ik}^2
\to
\frac{1}{n_i (n_i - 1)} \sum_{j=1}^{n_i} \sum_{j'=1: j' \neq j}^{n_i} y_{ik}^{(j)} y_{ik}^{(j')},
\end{align}
where an expectation of each of the summand of {\it r.h.s.} is $\E[Q_k^2]$,
hence that of {\it r.h.s.} is also be $\E[Q_k^2]$.
Consequently, the difference of 
the plugin estimator $\aEL_k$ 
and 
the unbiased estimator $\hEL_k$ 
is calculated as follows:
\begin{align*}
\aEL_k - \hEL_k
&= \frac{1}{N} \sum_i \left[
  \hmu_i^2
  - \frac{1}{n_i (n_i - 1)} \sum_{j=1}^{n_i} \sum_{j'=1: j' \neq j}^{n_i} y_{ik}^{(j)} y_{ik}^{(j')}
\right]
\\
&= \frac{1}{N} \sum_i \left\{
  \hmu_i^2
  - \frac{1}{n_i (n_i - 1)} \left[ \pars{\sum_{j=1}^{n_i} y_{ik}^{(j)}}^2 - \sum_{j=1}^{n_i} y_{ik}^{(j)}
 \right]
\right\}
\\
&= \frac{1}{N} \sum_i \left\{
  \hmu_i^2
  - \frac{1}{n_i (n_i - 1)} \left( n_i^2 \hmu_{ik}^2 - n_i \hmu_{ik} \right)
\right\}
\\
&= \frac{1}{N} \sum_i \frac{1}{n_i - 1} \hmu_{ik} \pars{1 - \hmu_{ik}}.
\end{align*}
\end{proof}

\subsection{Debiased Estimators of CL}
In this section, we give a proof for Prop. \ref{prop:cl_debias}.
\begin{define}[Plugin estimator of $\CL$]
\begin{align}
\aCL_{kb}(\mathcal{B}_k) &:= \frac{\abs{I_{kb}}}{N} (\bar{c}_{kb} - \bar{z}_{kb})^2,
\quad \text{where} \quad
\bar{c}_{kb} := \frac{\sum_{i \in I_{kb}} \hmu_{ik}}{\abs{I_{kb}}},
\quad
\bar{z}_{kb} := \frac{\sum_{i \in I_{kb}} z_{ik}}{\abs{I_{kb}}},
\end{align}
where $I_{kb} := \{ i \,|\, z_{ik} \in \mathcal{B}_{kb} \}$ 
denotes an index set of $b$-th bin
and $\mathcal{B}_{kb} := [\zeta_{kb}, \zeta_{kb+1})$ is a $b$-th interval of the binning scheme $\mathcal{B}_k$.
\end{define}
\begin{repprop}{prop:cl_debias}[Debiased estimator of $\CL_{kb}$]
\label{prop:cl_debias2}
The plugin estimator of $\CL_{kb}$
is debiased with the following estimator:
\begin{align}
\hCL_{kb}(\mathcal{B}_k) &:=
\aCL_{kb}(\mathcal{B}_k) -
\frac{\abs{I_{kb}}}{N} \frac{\bar{\sigma}_{kb}^2}{\abs{I_{kb}} - 1},
\quad \text{where} \quad
\bar{\sigma}_{kb}^2 := \frac{1}{\abs{I_{kb}}} \sum_{i \in I_{kb}} \hmu_{ik}^2 - \pars{\frac{1}{\abs{I_{kb}}} \sum_{i \in I_{kb}} \hmu_{ik}}^2.
\end{align}
\end{repprop}
\begin{proof}
The bias of the plugin estimator $\aCL_{kb}$ is explained in a similar manner as in the case of $\aEL_k$.
Concretely, a bias of the term $\bar{c}_{kb}^2$ for an estimation of $\bar{C}_{kb}^2$ can be reduced with a following replacement:
\begin{align}
\bar{c}^2_{kb} \to \frac{1}{\abs{I_{kb}}(\abs{I_{kb}} - 1)} \sum_{i \in I_{kb}} \sum_{i' \in I_{kb} : i' \neq i} \hmu_{ik} \hmu_{i'k}.
\label{eq:calib_replace}
\end{align}
Note that the {\it r.h.s.} term is only defined for the bin with $\abs{I_{kb}} > 1$.
In this case, a conditional expectation of the term is as follows:
\begin{align*}
\E[ \frac{\abs{I_{kb}}}{N} \cdot r.h.s., \abs{I_{kb}} > 1 ]
&= \sum_{m=2}^{N} 
\frac{\E[ \abs{I_{kb}} = m ] m}{N}
\frac{1}{m(m - 1)} \sum_{i \in I_{kb}} \sum_{i' \in I_{kb} : i' \neq i} 
\E[ \hmu_{ik} \hmu_{i'k} | \abs{I_{kb}} = m ]
\\
&= \sum_{m=2}^{N} 
\frac{\E[ \abs{I_{kb}} = m ] m}{N}
\bar{C}_{kb}^2
= \frac{\E[\abs{I_{kb}} | \abs{I_{kb}} > 1]}{N}
\bar{C}_{kb}^2
\\
&= \pars{\E[Z_k \in \mathcal{B}_{kb}] - \eta_{kb}}
\bar{C}_{kb}^2,
\end{align*}
where $\eta_{kb} = \E[I_{kb} \leq 1]/ N$,
which can be reduced by increasing $N$ relative to the bin size.
When we use the $r.h.s. = 0$ for $\abs{I_{kb}} \leq 1$,
$\eta_{kb} \bar{C}_{kb}^2$ is a remained bias term after 
the replacement in equation \eqref{eq:calib_replace}.
When we define an estimator $\hCL_{kb}$ as a modified $\aCL_{kb}$
that has been applied the replacement \eqref{eq:calib_replace},
a debiasing amount of the bias with the modification is calculated as follows:
\begin{align*}
\aCL_{kb} - \hCL_{kb}
&= \frac{\abs{I_{kb}}}{N} \left\{
  \bar{c}^2_{kb} - \frac{1}{\abs{I_{kb}}(\abs{I_{kb}} - 1)} \sum_{i \in I_{kb}} \sum_{i' \in I_{kb} : i' \neq i} \hmu_{ik} \hmu_{i'k}
\right\}
\\
&= \frac{\abs{I_{kb}}}{N} \left\{
   \bar{c}^2_{kb} -
  \frac{1}{\abs{I_{kb}}(\abs{I_{kb}} - 1)} \left[
    \pars{\abs{I_{kb}} \bar{c}_{kb}}^2 - \sum_{i \in I_{kb}} \hmu_{ik}^2
  \right]
\right\}
\\
&= \frac{\abs{I_{kb}}}{N} \left\{
  \frac{1}{\abs{I_{kb}} - 1} \left[
  \frac{1}{\abs{I_{kb}}} \sum_{i \in I_{kb}} \hmu_{ik}^2
  - \bar{c}_{kb}^2
  \right]
\right\}
= \frac{\abs{I_{kb}}}{N}
  \frac{\bar{\sigma}_{kb}^2}{\abs{I_{kb}} - 1}.
\end{align*}
\end{proof}
Note that as we mentioned in the proof, the bin-wise debiasing cannot be applied for the bins with $\abs{I_{kb}} \leq 1$.
We use $0$ for the estimators with such bins.
For single-labeled data,
the remaining bias from this limitation is also analyzed in the literature \cite{ferro2012bias}.

\subsection{Definition and estimators of dispersion loss}
\label{aps:DL}

We consider to estimate the remainder term $\EL - \CL$.
As we present in equation \eqref{eq:decomp2},
$\EL$ is decomposed into $\CL + \GL$,
in which $\GL$ is a loss relating to the lack of predictive sharpness.
However, the approximate calibration loss $\CL(\mathcal{B})$
is known to be underestimated \cite{vaicenavicius2019evaluating,kumar2019verified}
in relation to the coarseness of the selected binning scheme $\mathcal{B}$.
On the other hand, $\EL$ does not suffer from a resolution of $\mathcal{B}$.
Instead of estimating the $\GL$ term for binned predictions with $\mathcal{B}$,
we use the difference term $\DL(\mathcal{B}) := \EL - \CL(\mathcal{B})$,
which we call dispersion loss.
The non-negativity of $\DL(\mathcal{B})$ is shown as follows.
\begin{prop}[Non-negativity of dispersion loss]
Given a binning scheme $\mathcal{B}$, a dispersion loss
for class $k$ is decomposed into bin-wise components, where each term takes a non-negative value:
\begin{align}
\DL_k(\mathcal{B})
&:= \EL_k - \CL_k(\mathcal{B})
= \sum_b \DL_{kb}(\mathcal{B}),
\\
\DL_{kb}(\mathcal{B})
&:= \E[\E[\{(Q_k - \bar{C}_{kb}) - (Z_{k} - \bar{Z}_{kb})\}^2 | Z_{k} \in \mathcal{B}_{kb}]] \geq 0.
\label{eq:dl_form}
\end{align}
\end{prop}
\begin{proof}
From the definition of $\DL_k(\mathcal{B})$,
\begin{align*}
\DL_{k}(\mathcal{B})
&:= \EL_k - \CL_k(\mathcal{B})
\\
&= \E[(Q_k - Z_k)^2] - \sum_b \E[\E[(\bar{C}_{kb} - \bar{Z}_{kb})^2 
\,|\, Z_{k} \in \mathcal{B}_{kb}]]
\\
&= \sum_b \E[\E[(Q_k - Z_k)^2 - (\bar{C}_{kb} - \bar{Z}_{kb})^2 
\,|\, Z_{k} \in \mathcal{B}_{kb}]]
\\
&= \sum_b \DL_{kb}(\mathcal{B}),
\\
&\quad \text{where} \quad
\DL_{kb}(\mathcal{B}) := \E[\E[(Q_k - Z_k)^2 - (\bar{C}_{kb} - \bar{Z}_{kb})^2 
\,|\, Z_{k} \in \mathcal{B}_{kb}]].
\end{align*}
By noting that
$\bar{C}_{kb} - \bar{Z}_{kb}
= \E[Q_{k} - Z_{k} | Z_{k} \in \mathcal{B}_{kb}]$,
the last term is further transformed
as follows:
\begin{align*}
\DL_{kb}(\mathcal{B}) 
&= \E[\E[(Q_k - Z_k)^2 - (\bar{C}_{kb} - \bar{Z}_{kb})^2 
\,|\, Z_{k} \in \mathcal{B}_{kb}]]
\\
&= \E[\E[(Q_k - Z_k)^2 
- 2 (Q_k - Z_k) (\bar{C}_{kb} - \bar{Z}_{kb})
+ (\bar{C}_{kb} - \bar{Z}_{kb})^2
\,|\, Z_{k} \in \mathcal{B}_{kb}]]
\\
&= \E[\E[
\{(Q_k - Z_k) - (\bar{C}_{kb} - \bar{Z}_{kb})\}^2
\,|\, Z_{k} \in \mathcal{B}_{kb}]]
\\
&= \E[\E[
\{(Q_k - \bar{C}_{kb}) - (Z_k - \bar{Z}_{kb})\}^2
\,|\, Z_{k} \in \mathcal{B}_{kb}]].
\end{align*}
Then, the last term is apparently $\geq 0$.
\end{proof}
From equation \eqref{eq:dl_form},
the DL term can be interpreted as the average of the bin-wise overdispersion of the true class probability $Q_k$, which is unaccounted for by the deviation of $Z_k$.
For single-labeled cases,
similar argument is found in \cite{stephenson2008two}.
The plugin and debiased estimators of $\DL$ are derived from those of $\EL$ and $\CL$, respectively.

By using the plugin and the debiased estimators of $\EL$ and $\CL$,
those estimators of $\DL$ are defined as follows:
\begin{define}[Plugin / debiased estimators of dispersion loss]
\begin{align}
\aDL_{kb}(\mathcal{B})
&= \frac{1}{N} \sum_{i \in I_{kb}} \{
(\hmu_{ik} - \bar{c}_{kb}) - (z_{ik} - \bar{z}_{kb})
\}^2,
\\
\hDL_{kb}(\mathcal{B})
&= \aDL_{kb}(\mathcal{B}) 
- \frac{1}{N} \sum_{i \in I_{kb}} \left(
  \frac{1}{n_i-1} \hmu_{ik} (1 - \hmu_{ik})
  - \frac{\bar{\sigma}_{kb}^2}{\abs{I_{kb}} - 1}
  \right).
\label{eq:ap_dl_debias}
\end{align}
\end{define}

\subsection{Experimental setup for debiasing effects of {$\EL$} and  {$\CL$} terms}
\label{aps:el_cl_debias_details}

The details of the experimental setup for Section \ref{sec:el_cl_debias} are described.
We experimented on evalutions of a perfect predictor that indicated correct instance-wise CPEs,
using synthetic binary labels with varying instance sizes: from $100$ to $10,000$.
For each instance, the positive probability for label generation
was drawn from a uniform distribution $U(0, 1)$,
and two or five labels were generated in an i.i.d. manner
following a Binomial distribution with the corresponding probability.
Since the predictor indicated the correct probability,
$\EL$ and $\CL$ would be zero in expectation.
For a binning scheme $\mathcal{B}$ of the estimators, we adopted $15$ equally-spaced binning,
which was regularly used to evaluate calibration errors \cite{guo2017calibration}.

\section{Details on higher-order statistics evaluation}
\label{aps:higher-order}

The details and proofs for the statements in section \ref{sec:higher-order} are described.
Let $X \in \mathcal{X}$ be an input feature and $\{Y^{(j)} \in e^K \}_{j=1}^n$ be $n$ distinct labels for the same instance.
We define a symmetric categorical statistics $\phi: e^{K \times n} \to e^M (M \geq 2)$ for the $n$ labels.
For the case of $M=2$, $\phi$ can be equivalently represented as $\phi: e^{K \times n} \to \{0, 1\}$, and we use this definition for the successive discussion.
In our experiments,
we particularly focus on a disagreement between paired labels $\phi^D = \mathbb{I}[Y^{(1)} \neq Y^{(2)}]$ 
as predictive target.

Consider a probability prediction $\varphi: \mathcal{X} \to [0, 1]$ for statistics $\phi: e^{K \times n} \to \{0, 1\}$,
a strictly proper loss $\ell: \{0, 1\} \times [0, 1] \to \mathbb{R}$
encourages $\varphi(X)$ to approach the right probability $P(\phi(Y^{(1)}, \dots, Y^{(n)}) | X)$ in expectation.
We use (one dimensional) squared loss $\ell(\phi, \varphi) = (\phi - \varphi)^2$ in our evaluation.
The expected loss is as follows:
\begin{define}[Expected squared loss for $\phi$ and $\varphi$]
\begin{align}
L_\phi := \E[(\phi - \varphi)^2],
\end{align}
where the expectation is taken over the random variables $X$ and $Y^{(1)}, \dots, Y^{(n)}$.
\end{define}
Note that $L_\phi$ for an empirical distribution is equivalent to Brier score of $\phi$ and $\varphi$.
A decomposition of $L_{\phi}$ into $\CL_\phi$ and $\RL_\phi$ 
is readily available by applying Theorem \ref{th:decomp_cl_rl}.
\begin{align}
L_{\phi} &:= \E[(\phi - \varphi)^2]
= \underset{\CL_{\phi}}{\underbrace{\E[(\E[\phi|\varphi] - \varphi)^2]}}
+ \underset{\RL_{\phi}}{\underbrace{\E[(\phi - \E[\phi|\varphi])^2]}}.
\end{align}

We will derive the estimators of $L_{\phi}$ and $\CL_\phi$ as evaluation metrics.
However, the number of labels per instance is $n_i \geq n$ 
in general\footnote{We omit an instance with $n_i < n$ where $\phi$ cannot be calculated with distinct $n$ labels.},
which results in multiple inconsistent statistics $\phi$ for the same instance.
The problem can be solved with similar treatments as in the evaluation of CPEs.

As we stated in section \ref{sec:higher-order},
an unbiased estimator of the mean statistics for each instance $\mu_{\phi,i} := \E[\phi | X = x_i]$ is a useful building block in the estimation of {$L_{\phi}$} and $\CL_{\phi}$.
Recall that we assume
a conditional independence of an arbitrary number of labels given an input feature, {\it i.e.}, $Y^{(1)}, \dots, Y^{(n_i)} | X \underset{i.i.d.}{\sim} \mathrm{Cat}(Q(X))$,
$\mu_{\phi,i}$ is estimated as follows:
\begin{theorem}[Unbiased estimator of $\mu_{\phi,i}$]
For an instance $i$ with $n_i$ labels obtained in a conditional i.i.d. manner, an unbiased estimator of the conditional mean $\mu_{\phi,i}$ is given as follows:
\begin{align}
\hmu_{\phi,i}
&:= \binom{n_i}{n}^{-1}
    \sum_{j \in \mathrm{Comb}(n_i, n)} \phi({y}_i^{(j_1)}, \dots, {y}_i^{(j_n)}),
\label{eq:phi_mu}
\end{align}
where $\mathrm{Comb}(n_i, n)$ denotes the distinct subset of size $n$ drawn from $\{1, \dots, n_i\}$ without replacement.
\end{theorem}
\begin{proof}
This is directly followed from the fact that $\hmu_{\phi,i}$ is a U-statistic of $n$-sample symmetric kernel function $\phi$ \cite{hoeffding1948class}.
\end{proof}

\subsection{Unbiased estimator of  {$L_{\phi}$}}

We give an unbiased of $L_\phi$ as follows:
\begin{theorem}[Unbiased estimator of {$L_{\phi}$}]
The following estimator is an unbiased estimator of $L_{\phi}$.
\begin{align}
\hL_{\phi}
&:= \frac{1}{N} \sum_i \hL_{\phi,i},
\quad \text{where} \quad
\hL_{\phi,i} := \binom{n_i}{n}^{-1} \sum_{j \in \mathrm{Comb}(n_i, n)} \pars{\phi({y}_i^{(j_1)}, \dots, {y}_i^{(j_n)}) - \varphi_{i}}^2.
\end{align}
\end{theorem}
\begin{proof}
We first confirm that, for each random variables $\phi(y_i^{(1)}, \dots, y_i^{(n)})$ and $\varphi_i$ of sample $i$,
\begin{align*}
\E[f(y_i^{(1)}, \dots y_i^{(n)}; \varphi_i) ] = L_{\phi},
\quad \text{where} \quad
f(y_i^{(1)}, \dots y_i^{(n)}; \varphi_i)
:= \pars{\phi(y_i^{(1)}, \dots, y_i^{(n)}) - \varphi_i}^2,
\end{align*}
is satisfied by definition.
As $f$ is an $n$-sample symmetric kernel of variables $y_i^{(1)}, \dots, y_i^{(n)}$,
$\hL_{\phi,i}$ is a U-statistic \cite{hoeffding1948class} of the kernel given $\varphi_i$ and also an unbiased estimator of $L_{\phi}$ as follows:
\begin{align}
\E[ \hL_{\phi,i} ]
=
\E[\E[ \hL_{\phi,i} | \varphi_i ]]
=
\E[\E[f(y_i^{(1)}, \dots y_i^{(n)}; \varphi_i) | \varphi_i ]]
= L_{\phi}.
\end{align}
Hence
\begin{align*}
\E[\hL_{\phi}] 
= \frac{1}{N} \sum_{i=1}^N \E[\hL_{\phi,i}] = L_{\phi}.
\end{align*}
\end{proof}

\subsection{Debiased estimator of $\CL_\phi(\mathcal{B})$}
Following the same discussion as the $\CL(\mathcal{B})$ term of CPEs,
we also consider a binning based approximation of $\CL_\phi$
stratified with a binning scheme $\mathcal{B}$ for predictive probability $\varphi \in [0, 1]$.
Then, the plugin estimator of $\CL_{\phi}(\mathcal{B})$ 
is defined as follows:
\begin{define}[Plugin estimator of $\CL_{\phi}$]
\begin{align}
\aCL_{\phi}(\mathcal{B}) 
&:= \sum_{b=1}^B \aCL_{\phi,b}(\mathcal{B}),
\\
\quad &\text{where} \quad
\aCL_{\phi,b}(\mathcal{B})
:= \frac{\abs{I_{\phi,b}}}{N} (\bar{c}_{\phi,b} - \bar{\varphi}_b)^2,
\quad
\bar{c}_{\phi,b} := \frac{\sum_{i \in I_{\phi,b}} \hmu_{\phi,i}}{\abs{I_{\phi,b}}},
\quad
\bar{\varphi}_{\phi,b} := \frac{\sum_{i \in I_{\phi,b}} \varphi_i}{\abs{I_{\phi,b}}}.
\end{align}
\end{define}
We again improve the plugin estimator with
the following debiased estimator $\hCL_{\phi,b}(\mathcal{B})$:
\begin{corollary}[Debiased estimator of $\CL_{\phi,b}$]
A plugin estimator $\aCL_{\phi,b}(\mathcal{B})$ of $\CL_{\phi,b}(\mathcal{B})$
is debiased to $\hCL_{\phi,b}(\mathcal{B})$ with a correction term as follows:
\begin{align}
\hCL_{\phi,b}(\mathcal{B})
&:=
\aCL_{\phi,b}(\mathcal{B})
- \frac{\abs{I_{\phi,b}}}{N} \frac{\bar{\sigma}_{\phi,b}^2}{\abs{I_{\phi,b}} - 1},
\\
\quad &\text{where} \quad
\bar{\sigma}_{\phi,b}^2
:= \frac{1}{\abs{I_{\phi,b}}} \sum_{i \in I_{\phi,b}} \hmu_{\phi,i}^2
- \pars{\frac{1}{\abs{I_{\phi,b}}} \sum_{i \in I_{\phi,b}} \hmu_{\phi,i}}^2.
\end{align}
Note that the estimator is only available for bins with $\abs{I_{\phi,b}} \geq 2$.
\end{corollary}
\begin{proof}
The proof follows a similar reasoning to Prop. \ref{prop:cl_debias2}.
We reduce the bias introduced with the term $\bar{c}_{\phi,b}^2$
by replacing the term with unbiased one for $\bar{C}_{\phi,b}^2 = \E[\phi | \varphi \in \mathcal{B}_{b}]^2$
as follows:
\begin{align*}
\bar{c}^2_{\phi,b} \to \frac{1}{\abs{I_{\phi,b}}(\abs{I_{\phi,b}} - 1)} \sum_{i \in I_{\phi,b}} \sum_{i' \in I_{\phi,b} : i' \neq i} \hmu_{\phi,i} \hmu_{\phi,i'},
\end{align*}
where $\hmu_{\phi,i}$ is defined in equation \eqref{eq:phi_mu}.
An improvement with the debiased estimator $\aCL_{\phi,b} - \hCL_{\phi,b}$
is also calculated with the same manner as in Prop. \ref{prop:cl_debias2}.
\end{proof}

\subsection{Summary of evaluation metrics introduced for label histograms}

In Table \ref{tb:evaluation-metrics},
we summarize evaluation metrics introduced for label histograms,
where {\it order} shows the required numbers of labels for each instance to define the metrics,
and {\it rater} represents those for estimating the metrics.

\begin{table*}[h!]
  \caption{Summary of evaluation metrics introduced for label histograms}
  \label{tb:evaluation-metrics}
  \vspace{.2cm}
  \centering
  \begin{tabular}{cllc}
    \toprule
    Order & Signature & Description & Rater \\
    \midrule
    $1$ & $\LPS = \EL + \IL$ & Expected squared loss of CPEs & $\geq 1$ \\
        & $\EL = \CL + \DL$ & Epistemic loss of CPEs & $\geq 2$ \\
        & $\CL = \CE^2$ & Calibration loss of CPEs & $\geq 1$ \\
        & $\DL$ & Dispersion loss of CPEs & $\geq 2$ \\
    \midrule
    $2$ & $L_{\phi^{\mathrm{D}}}$
				& Expected squared loss of DPEs & $\geq 2$
      \\
        & $\CL_{\phi^{\mathrm{D}}}$
				& Calibration loss of DPEs & $\geq 2$
			\\
    \bottomrule
  \end{tabular}
\end{table*}

\section{Details on post-hoc uncertainty calibration methods}

\subsection{CPE calibration methods based on linear transformations}
\label{aps:cpe_calib}

To complement Section \ref{sec:bg_calibration},
we summarize the formulation of CPE (class probability estimation) calibration that is based on linear transformations.
Let $x \in \mathcal{X}$ denotes an input data,
$u: \mathcal{X} \to \mathbb{R}^K$ denotes a DNN function that outputs a logit vector,
and $f(x) = \mathrm{softmax}(u(x)) \in \Delta^{K-1}$ denotes CPEs.
A common form of CPE calibration with linear transformations is given as follows:
\begin{align}
\widetilde{u}(x) &= W u(x) + b,
\label{eq:matrix_scaling}
\\
\widetilde{f}(x) &= \mathrm{softmax}(\widetilde{u}(x)),
\end{align}
where $\widetilde{u}(x)$ denotes transformed logits with parameters $W \in \mathbb{R}^{K \times K}$ and $b \in \mathbb{R}^K$,
and $\tilde{f}: \mathcal{X} \to \Delta^{K-1}$ denotes CPEs after calibration.

The most general form of equation \eqref{eq:matrix_scaling} is referred to as matrix scaling \cite{guo2017calibration,kull2019beyond}.
A version of that with a constraint $W = \mathrm{diag}(v), \, v \in \mathbb{R}^K$ 
and that with a further constraint $v = 1/t, \, t \in \mathbb{R}, b = 0$
are called vector and temperature scaling, respectively.
In particular, temperature scaling has a favorable property; it does not change the maximum predictive class of each instance, and hence neither the overall accuracy,
as the order of vector elements between $u$ and $\tilde{u}$ for each $x$ is unchanged.

For vector and matrix scaling, regularization terms are required to prevent over-fitting;
L2 regularization of $b$:
\begin{align}
\Omega_{\mathrm{L2}}(b) := \lambda_b \frac{1}{K} \sum_k b_k^2
\label{eq:calib_l2_reg}
\end{align}
is commonly used for vector scaling,
and off-diagonal and intercept regularization (ODIR):
\begin{align}
\Omega_{\mathrm{ODIR}}(W, b) := \lambda_w \frac{1}{K(K-1)} \sum_{k \neq k'} W_{kk'}^2
+ \lambda_b \frac{1}{K} \sum_k b_k^2
\label{eq:calib_odir}
\end{align}
is proposed for matrix scaling, which is used for improving class-wise calibration \cite{kull2019beyond}.

\subsection{Details on  {$\alpha$}-calibration}
\begin{figure*}[!t]
  \centering
  \rule[-.0cm]{0cm}{0cm}
  \begin{tabular}{cc}
      \includegraphics[width=1.0\linewidth]{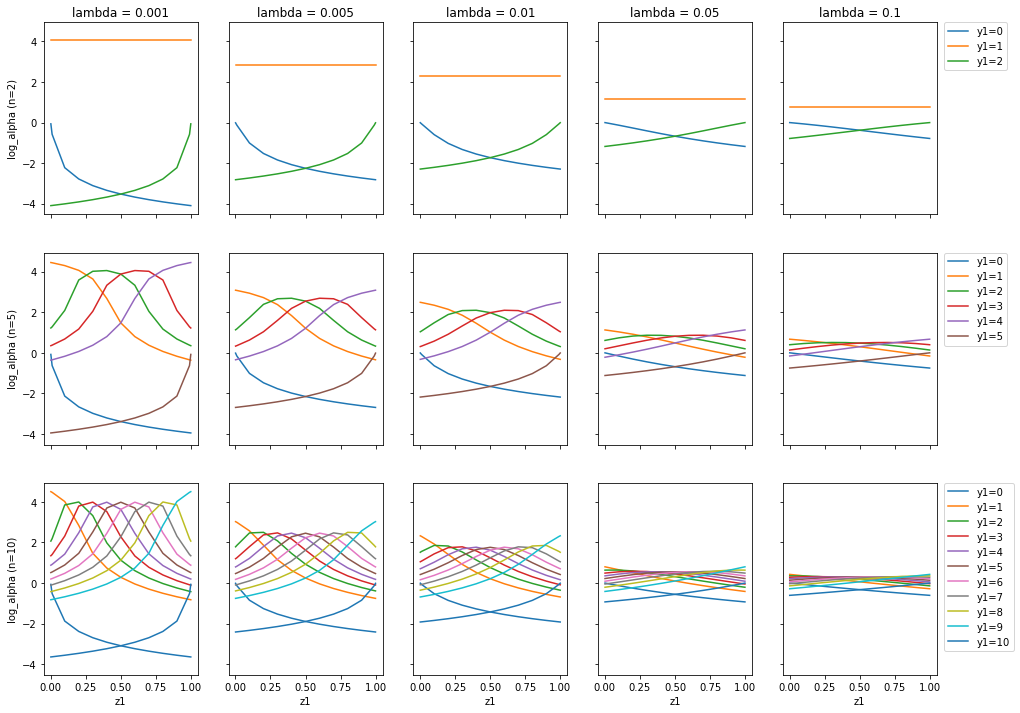}
      \\
      Optimal values of $\log \alpha_0$ for different values of the hyperparameter $\lambda_\alpha$
  \end{tabular}
  \rule[-.0cm]{4cm}{0cm}
  \caption{
    Optimal values of $\log \alpha_0(x_i)$ are numerically evaluated for binary class problems with the number of labels $n_i \in \{2, 5, 10\}$, label histograms $y_i$ with $0 \leq y_{i1} \leq n_i$,
    class probability estimations of the first class: $z1 \in \{.001, .01, .1, .2, \dots, .8, .9, .99, .999\}$,
    and the hyperparameter $\lambda_\alpha$ (Appendix \ref{aps:hypara_alpha0}).
    As expected, the range of $\log \alpha_0$ contains zero and gets narrower as $\lambda_\alpha$ increases.
    }
\label{fig:lambda_alphas}
\end{figure*}

\label{aps:alpha_calib}

\paragraph{Loss function}

For the loss function for $\alpha$-calibration, we use a variant of $\mathrm{NLL}$ in equation \eqref{eq:calib_nll}
as follows:
\begin{align}
- \frac{1}{\sum_i n_i} \sum_{i} \log \mathrm{DirMult}(y_i | \alpha_0(x_i) f(x_i))
+ \frac{\lambda_{\alpha}}{N} \sum_i (\log \alpha_0(x_i))^2,
\label{eq:alpha_loss}
\end{align}
where
$\mathrm{DirMult}(\cdot)$ denotes the Dirichlet multinomial distribution,
and a regularization term $\lambda_{\alpha} (\log \alpha_0)^2$ 
is introduced for stabilization purpose, which penalizes the deviation from $\alpha_0=1$ to both directions towards extreme concentrations of the mass: $P(\zeta | X) \to \delta(\zeta = f(X))$ with $\alpha_0 \to \infty$ or $P(\zeta | X) \to \sum_k \delta(\zeta = e_k) \E[\zeta_k]$ with $\alpha_0 \to 0$.
We employ $\lambda_{\alpha} = 0.005$ throughout this study.

\paragraph{Hyperparameter analysis for the optimal values of $\alpha_0$}
\label{aps:hypara_alpha0}

Intuitively, $\log \alpha_0$ is likely to get close to zero as the regularization coefficient $\lambda_{\alpha}$ increases.
If we regard $\log \alpha_0(x_i)$ as a free parameter,
the optimal value of $\log \alpha_0(x_i)$ only depends on the CPEs $f(x_i)$, the number of labels $n_i$ and the observed labels $y_i$ for each instance.
We assume that the number of labels $n_i$ is common for all the instances for simplicity.
The optimality condition for $\alpha_0$ is obtained 
by taking a derivative of equation \eqref{eq:alpha_loss} with respect to $\log \alpha_0(x_i)$ as follows:
\begin{align}
  0 &= - \frac{\alpha_0}{n_i} \left[ 
    \sum_k f_k \pars{ \psi(\alpha_0 f_k + y_{ik}) - \psi(\alpha_0 + n_i) }
    - \sum_k f_k \pars{ \psi(\alpha_0 f_k) - \psi(\alpha_0) }
  \right] + 2 \lambda_\alpha \log \alpha_0
  \nonumber
  \\
  &= - \frac{\alpha_0}{n_i} \pars{
    \sum_k \sum_{l=1}^{y_{ik}} \frac{f_k}{\alpha_0 f_k + l - 1}
    - \sum_{l=1}^{n_i} \frac{1}{\alpha_0 + l - 1}
  }
  + 2 \lambda_\alpha \log \alpha_0
  ,
  \label{eq:alpha_optimality}
\end{align}
where $\psi(\cdot)$ denotes the digamma function,
and a recurrence formula $\psi(s + 1) = \psi(s) + \frac{1}{s}$ is used for the derivation.

One can verify that divergences of the optimal $\log \alpha_0$ occur in some special cases with $\lambda_\alpha = 0$.
For example, if the labels are unanimous, i.e., $y_{i1} = n_i$, $n_i > 1$, and $f_1 < 1$, the r.h.s. of equation \eqref{eq:alpha_optimality} turns out to be positive as follows:
\begin{align*}
  - \frac{\alpha_0}{n_i} \sum_{l=1}^{n_i}  \pars{
      \frac{f_1}{\alpha_0 f_1 + l - 1} - \frac{1}{\alpha_0 + l - 1}
    } > 0,
\end{align*}
which implies that $\log \alpha_0 \to -\infty$.
In contrast, if $n_i = 2$, $K=2$, and $y_{i1} = y_{i2} = 1$,
the r.h.s of equation \eqref{eq:alpha_optimality} is calculated as follows:
\begin{align*}
  - \frac{\alpha_0}{n_i} \pars{\frac{2}{\alpha_0} - \frac{1}{\alpha_0} - \frac{1}{\alpha_0 + 1}}
  = - \frac{1}{n_i (\alpha_0 + 1)} < 0,
\end{align*}
which results in $\log \alpha_0 \to \infty$.

For a finite $\lambda_\alpha > 0$,
the optimal values of $\log \alpha_0$ can be numerically evaluated by Newton's method. 
We show these values in Fig. \ref{fig:lambda_alphas} for several conditions of binary class problems.
As expected, the range of the optimal $\log \alpha_0$ contains zero and gets narrower as $\lambda_\alpha$ increases.

\subsection{Proof for Theorem \ref{th:alpha_calib_analysis}}
\label{aps:alpha_calib_analysis_proof}

\begin{reptheorem}{th:alpha_calib_analysis}
There exist intervals for parameter $\alpha_0 \geq 0$, %
which improve task performances as follows.
\begin{enumerate}
\item For DPEs, $L_{\phi^{\mathrm{D}},G} \leq L_{\phi^{\mathrm{D}},G}^{(0)}$ holds when $(1 - 2u_Q + s_Z)/2(u_Q - s_Z) \leq \alpha_0$,
and $L_{\phi^{\mathrm{D}},G}$ takes the minimum value
when $\alpha_0 = (1 - u_Q)/(u_Q - s_Z)$,
if $u_Q > s_Z$ is satisfied.
\item For posterior CPEs, $\EL_G' \leq \EL_G'^{(0)}$ holds when $(1 - u_Q - \EL_G) / 2 \EL_G \leq \alpha_0$,
and $\EL_G'$ takes the minimum value when $\alpha_0 = (1 - u_Q) / \EL_G$,
if $\EL_G > 0$ is satisfied.
\end{enumerate}
Note that we denote $s_Z := \sum_k Z_k^2$, $u_Q := \E[\sum_k Q_k^2 | G]$, $v_Q := \V[Q_k | G]$,
and $\EL_G := \E[\sum_k (Z_k - Q_k)^2 | G]$.
The optimal $\alpha_0$ of both tasks coincide to be $\alpha_0 = (1 - u_Q) / v_Q$,
if CPEs match the true conditional class probabilities given $G$, {\it i.e.}, $Z = \E[Q | G]$.
\end{reptheorem}

\begin{proof}
We will omit the superscript $\mathrm{D}$ from $\phi^{\mathrm{D}}$ and $\varphi^{\mathrm{D}}$ for brevity.
First, we rewrite $\varphi$ and $Z'$ in equation \eqref{eq:alpha_dpe_posterior} as follows:
\begin{align}
\varphi
&= \gamma \pars{1 - \sum_k Z_k^2},
\quad
Z'_k
= \gamma Z_k + (1 - \gamma) Y_k,
\end{align}
where 
$\gamma := \alpha_0 / (\alpha_0 + 1)$. 
Note that $\gamma \in (0, 1)$ since $\alpha_0 \in (0, +\infty)$.
We also introduce the following variables:
\begin{align}
s_Q &:= \sum_k \E[Q_k | G]^2,
&
\bar{s}_Q &:= 1 - s_Q,
&
v_Q &:= \sum_k \V[Q_k | G],
\\
u_Q &:= \sum_k \E[Q_k^2 | G] = s_Q + v_Q,
&
\bar{u}_Q &:= 1 - u_Q, %
\\
s_Z &:= \sum_k Z_k^2,
&
\bar{s}_Z &:= 1 - s_Z,
\end{align}
where, all the variables reside within $[0, 1]$ %
since $Z, Q \in \Delta^{K-1}$.

{\bf 1. The first statement: DPE}

The objective function to be minimized is as follows:
\begin{align}
L_{\phi,G} = \E[(\phi - \varphi)^2 | G]
&=
\pars{\E[\phi | G] - \varphi}^2 + \mathbb{V}[\phi | G]
\\
&= \E[(\bar{u}_Q - \gamma \bar{s}_Z)^2] + \mathbb{V}[\phi | G],
\end{align}
where we use the relation $\E[\phi|G] = \bar{u}_Q$ and $\varphi = \gamma \bar{s}_Z$.
Note that only the first term is varied with $\alpha_0$,
and $L_{\phi,G} \to L_{\phi,G}^{(0)} \, (\gamma \to 1)$.
The condition for satisfying $L_{\phi,G} \leq L_{\phi,G}^{(0)}$ is found by solving
\begin{align}
0 = L_{\phi_G} - L_{\phi,G}^{(0)}
= (\gamma - 1) \bar{s}_Z \{(\gamma+1) \bar{s}_Z - 2 \bar{u}_Q\}
= (\gamma - 1) \bar{s}_Z \{\gamma \bar{s}_Z + \bar{s}_Z - 2 \bar{u}_Q\}.
\end{align}
$\bar{s}_Z = 0$ and $\gamma \to 1$
are trivial solutions that correspond to
a hard label prediction ({\it i.e.}, $Z \in e^K$)
and $\alpha_0 \to \infty$, respectively.
The remaining condition for $L_{\phi,G} \leq L_{\phi,G}^{(0)}$ is 
\begin{align}
\gamma \in [-1 + 2\bar{u}_Q/\bar{s}_Z, 1),
\label{eq:alpha_theorem1_1_gamma_left}
\end{align}
which is feasible when $\bar{u}_Q < \bar{s}_Z$, {\it i.e.}, $u_Q > s_Z$.
In this case,
\begin{align}
\gamma^* = \bar{u}_Q/\bar{s}_Z
\label{eq:alpha_theorem1_1_gamma_opt}
\end{align}
is the optimal solution for $\gamma$.
By using a relation $\alpha_0 = \gamma / (1 - \gamma)$ with equations \eqref{eq:alpha_theorem1_1_gamma_left}
and \eqref{eq:alpha_theorem1_1_gamma_opt},
the first statement of the theorem is obtained as follows:
\begin{align}
\alpha_0 &\geq \frac{-\bar{s}_Z + 2\bar{u}_Q}{2\bar{s}_Z - 2\bar{u}_Q} = \frac{1 - 2u_Q + s_Z}{2(u_Q - s_Z)},
\quad
\alpha_0^* = \frac{\bar{u}_Q}{\bar{s}_Z - \bar{u}_Q} = \frac{1 - u_Q}{u_Q - s_Z}.
\end{align}
If $Z = \E[Q | G]$ is satisfied, $s_Z = s_Q$ holds, and the above conditions become as follows:
\begin{align}
\alpha_0 &\geq \frac{1 - u_Q - v_Q}{2v_Q},
\quad
\alpha_0^* = \frac{1 - u_Q}{v_Q}.
\label{eq:alpha_theorem1_1_calib_results}
\end{align}

{\bf 2. The second statement: posterior CPE}

The objective for the second problem is as follows:
\begin{align}
\EL'_G 
& = \E[\sum_k (Z'_k - Q_k)^2 |G]
\nonumber
\\
&= \E[\sum_k (\gamma Z_k + (1 - \gamma) Y_k - Q_k)^2 | G]
\nonumber
\\
&= \E[\sum_k ((Z_k - Q_k) + (1 - \gamma) (Y_k - Z_k))^2 | G]
\nonumber
\\
&= \sum_k \E[(Z_k - Q_k)^2 | G]
+ 2(1 - \gamma) \E[(Z_k - Q_k) (Y_k - Z_k) | G]
+ (1 - \gamma)^2 \E[(Y_k - Z_k)^2 | G],
\label{eq:alpha_thorem1_2_EL_dash}
\end{align}
where the first term equals to $\EL_G$, and the second and third term are further transformed as follows:
\begin{align}
\sum_k \E[(Z_k - Q_k) (Y_k - Z_k) | G]
&= \sum_k \E[\E[(Z_k - Q_k) (Y_k - Z_k) | Q] | G]
\nonumber
\\
&= - \sum_k \E[\E[(Z_k - Q_k)^2 | Q] | G] = - \EL_G,
\\
\sum_k \E[(Y_k - Z_k)^2 | G]
&= \sum_k \E[\E[(Y_k - Z_k)^2 | Q] | G]
\nonumber
\\
&= \sum_k \E[\E[(Y_k - 2 Y_k Z_k + Z_k^2) | Q] | G]
\nonumber
\\
&= \sum_k \E[(Q_k - 2 Q_k Z_k + Z_k^2) | G]
\nonumber
\\
&= \sum_k \E[(Q_k(1 - Q_k) + (Q_k - Z_k)^2 | G]
= \bar{u}_Q + \EL_G.
\end{align}
Hence equation \eqref{eq:alpha_thorem1_2_EL_dash} can be written as
\begin{align}
\EL'_G
&= 
\EL_G - 2(1 - \gamma) \EL_G + (1 - \gamma)^2 (\EL_G + \bar{u}_Q).
\end{align}
The condition for satisfying $\EL'_G \leq EL_G$ is obtained by solving
\begin{align}
0 = \EL'_G - \EL_G 
&= (1 - \gamma) \left\{(1 - \gamma) (\EL_G + \bar{u}_Q) - 2\EL_G \right\}.
\end{align}
If $\EL_G = 0$, which means $Z_k = Q_k$ given $G$, $\gamma \to 1$ is optimal as expected.
For the other case, {\it i.e.}, $\EL_G > 0$,
$\gamma$ that satisfying $\EL'_G \leq \EL_G$ and the optimal $\gamma$ are
\begin{align}
\gamma &\in \left[\frac{\bar{u}_Q - \EL_G}{\bar{u}_Q + \EL_G}, 1\right),
\quad
\gamma^* = \frac{\bar{u}_Q}{\bar{u}_Q + \EL_G},
\end{align}
respectively.
By using $\alpha_0 = \gamma/(1 - \gamma)$,
the corresponding $\alpha_0$ and $\alpha_0^*$ are
\begin{align}
\alpha_0 &\geq \frac{\bar{u}_Q - \EL_G}{2\EL_G} = \frac{1 - u_Q - \EL_G}{2\EL_G},
\quad
\alpha_0^* = \frac{\bar{u}_Q}{\EL_G} = \frac{1 - u_Q}{\EL_G},
\end{align}
respectively,
which are the second statement of the theorem.
If $Z = \E[Q | G]$ is satisfied,
$\EL_G = v_Q$ holds, and the above conditions become as follows:
\begin{align}
\alpha_0 &\geq \frac{1 - u_Q - v_Q}{2v_Q},
\quad
\alpha_0^* = \frac{1 - u_Q}{v_Q}.
\label{eq:alpha_theorem1_2_calib_results}
\end{align}
Notably, these are the same conditions as the terms in equation \eqref{eq:alpha_theorem1_1_calib_results}, respectively.
\end{proof}

\subsection{Summary of DPE computations}
\label{aps:deriv_dpes}

We use $\alpha$-calibration, ensemble-based methods (MCDO and TTA),
and a combination of them for 
predicting DPEs as follows. %

\paragraph{$\alpha$-calibration}
\begin{align}
\hvarphi^D
&= 1 - \sum_k \int \zeta_k^2 \, \mathrm{Dir}(\zeta | \alpha_0 f) \, d\zeta
= \frac{\alpha_0}{\alpha_0 + 1} \pars{1 - \sum_k f_k^2}.
\end{align}

\paragraph{Ensemble-based methods}
\begin{align}
\hvarphi^D &= 
\frac{1}{S} \sum_{s=1}^S \pars{1 - \sum_k \pars{f_k^{(s)}}^2},
\end{align}
where, $f^{(s)}: \mathcal{X} \to \Delta^{(k-1)}$ is the $s$-th prediction of the ensemble,
and $S$ is the size of the ensemble.

\paragraph{Ensemble-based methods with $\alpha$-calibration}

Although an output $\alpha$ already represents a CPE distribution without ensemble-based methods: MCDO and TTA, it can be formally combined with these methods.
In such cases, we calculate the predictive probability $\hvarphi$ as follows:
\begin{align}
\hvarphi^{\mathrm{D}} 
 = \frac{1}{S} \sum_{s=1}^S %
 \left[
	\frac{\alpha_0^{(s)}}{\alpha_0^{(s)} + 1} \pars{1 - \sum_k \pars{f_k^{(s)}}^2}
 \right],
\end{align}
where $\alpha_0^{(s)}, f^{(s)}$ denote the $s$-th ensembles of $\alpha_0$ and $f$, respectively.

\subsection{Summary of posterior CPE computations} %
\label{aps:deriv_post_cpes}

We consider a task for updating CPE of instance $x \in \mathcal{X}$ after an expert annotation $y \in e^K$.
For this task,
the posterior CPE distribution $P_{\mathrm{model}}(\zeta | x, y)$
is computed from an original (prior) CPE distribution model $P_{\mathrm{model}}(\zeta | x)$ as follows:
\begin{align}
P_{\mathrm{model}}(\zeta | x, y) = \frac{P_{\mathrm{model}}(y, \zeta | x)}{P_{\mathrm{model}}(y | x)},
\end{align}
where
\begin{align}
P_{\mathrm{model}}(y, \zeta | x)
= P_{\mathrm{model}}(\zeta | x) \prod_{k=1}^K \zeta_{k}^{y_k}.
\end{align}
For the case with multiple test instances, we assume that a predictive model is factorized as follows:
\begin{align}
P_{\mathrm{model},N}(\zeta_{1:N} | x_{1:N})
&= \prod_{i=1}^N P_{\mathrm{model}}(\zeta_i | x_i).
\end{align}
In this case, the posterior of CPEs is also factorized as follows:
\begin{align}
P_{\mathrm{model},N}(\zeta_{1:N} | x_{1:N}, y_{1:N})
&= \frac{P_{\mathrm{model},N}(y_{1:N}, \zeta_{1:N} | x_{1:N})}
{P_{\mathrm{model},N}(y_{1:N} | x_{1:N})}
= \prod_{i=1}^N \frac{P_{\mathrm{model}}(y_i, \zeta_i | x_i)}{P_{\mathrm{model}}(y_i | x_i)}
= \prod_{i=1}^N P_{\mathrm{model}}(\zeta_i | x_i, y_i).
\end{align}

\paragraph{$\alpha$-calibration}
Prior and posterior CPE distributions are computed as follows:
\begin{align}
P_{\alpha}(\zeta | x) &= \mathrm{Dir}(\zeta | \alpha_0(x) f(x)),
\\
P_{\alpha}(\zeta | x, y) &= \mathrm{Dir}(\zeta | \alpha_0(x) f(x) + y).
\end{align}

\paragraph{Ensemble-based methods}
Prior and posterior CPE distributions are computed as follows:
\begin{align}
P_{\mathrm{ens.}}(\zeta | x)
&= \frac{1}{S} \sum_{s=1}^S \delta(\zeta - f^{(s)}(x)),
\\
P_{\mathrm{ens.}}(\zeta | x, y)
&= \frac{1}{W'} \sum_{s=1}^S w'_s \delta(\zeta - f^{(s)}(x)),
\end{align}
where
$S$ is the size of the ensemble, $f^{(s)}$ denotes the $s$-th CPEs of the ensemble,
$w'_s := \sum_s \prod_k f^{(s)}(x)_k^{y_k}$,
and 
$W' := \sum_{s=1}^S w'_s$.

We omit the cases of predictive models combining the ensemble-based methods and $\alpha$-calibration,
where the posterior computation requires further approximation.

\subsection{Discussion on conditional i.i.d. assumption of label generations for  {$\alpha$}-calibration} %
\label{aps:discussion_cond_iid}

At the beginning of section \ref{sec:cpe_lh},
we assume a conditional i.i.d distribution for labels $Y$ given input data $X$,
which is also a basis for $\alpha$-calibration.
We expect that the assumption roughly holds in typical scenarios, where experts are randomly assigned to each example.
However, $\alpha$-calibration may not be suitable for
counter-examples that break the assumption.
For instance, if two fixed experts with different policy annotate all examples, these two labels would be highly correlated.
In such a case, the disagreement probability between them may be up to one and exceeds the maximum 
possible value $\varphi^{\mathrm{D}}$ allowed in equation \eqref{eq:alpha_dpe_posterior},
where $\varphi^{\mathrm{D}}$ always decreases from the original value, which corresponds to $\alpha_0 \to \infty$, by $\alpha$-calibration.

\section{Experimental details}
\label{aps:experimental_details}

We used the Keras Framework with Tensorflow backend \citeA{chollet2015keras} for implementation.

\subsection{Preprocessing}
\paragraph{Mix-MNIST and CIFAR-10}
We generated two synthetic image dataset: Mix-MNIST and Mix-CIFAR-10 from MNIST \cite{lecun2010mnist} and CIFAR-10 \cite{krizhevsky2009learning}, respectively.
We randomly selected half of the images to create mixed-up images from pairs and the other half were retained as original.
For each of the paired images, a random ratio that followed a uniform distribution $U(0, 1)$ was used for the mix-up and a class probability of multiple labels,
which were two or five in validation set.
For Mix-MNIST (Mix-CIFAR-10),
the numbers of generated instances were $37,500$ $(30,000)$, $7,500$ $(7,500)$, and $7,500$ $(7,500)$ for
training, validation, and test set, respectively.

\paragraph{MDS Data}
We used $80,610$ blood cell images with a size of $360 \times 363$, which was a part of the dataset obtained in a study of myelodysplastic syndrome (MDS) \cite{sasada2018inter},
where most of the images showed a white blood cell in the center of the image.
For each image, a mean of $5.67$ medical technologists
annotated the cellular category from $22$ subtypes, in which six were anomalous types.
We partitioned the dataset into training, validation, and test set with $55,356$, $14,144$, and $11,110$ images, respectively, where each of the partition consisted of images from distinct patient groups.
Considering the high labeling cost with experts in the medical domain,
we focused on scenarios that training instances were singly labeled,
and multiple labels were only available for validation and test set.
The mean number of labels per instance for validation and test set were $5.79$ and $7.58$, respectively.

\subsection{Deep neural network architecture}

\paragraph{Mix-MNIST}
For Mix-MNIST dataset, we used a neural network architecture with three convolutional and two full connection layers.
Specifically, the network had the following stack:
\begin{itemize}
\item Conv. layer with $32$ channels, $5\times5$ kernel, and ReLU activation
\item Max pooling with with $2\times2$ kernel and same padding
\item Conv. layer with $64$ channels, $7\times7$ kernel, and ReLU activation
\item Max pooling with with $4\times4$ kernel and same padding
\item Conv. layer with $128$ channels, $11\times11$ kernel, and ReLU activation
\item Global average pooling with $128$ dim. output
\item Dropout with $50\%$ rate
\item Full connection layer with $K=3$ dim. output and softmax activation
\end{itemize}

\paragraph{Mix-CIFAR-10 and MDS}
We adopted a modified VGG16 architecture \citeA{simonyan2014very} 
as a base model, in which 
the full connection layers were removed,
and the last layer was a max-pooling with $512$ output dimensions.
On top of the base model,
we appended the following layers:
\begin{itemize}
\item Dropout with $50\%$ rate and $512$ dim. output
\item Full connection layer with $128$ dim. output and ReLU activation
\item Dropout with $50\%$ rate and $128$ dim. output
\item Full connection layer with $22$ dim. output and softmax activation
\end{itemize}

\subsection{Training}
We used the following loss function for training:
\begin{align}
\mathcal{L}(y, z) = - \frac{1}{\sum_{i=1}^N n_i} \sum_{i=1}^N \sum_{k=1}^K y_{ik} \log z_{ik},
\label{eq:training_loss}
\end{align}
which was equivalent to the negative log-likelihood for instance-wise multinomial observational model except for constant.
We used Adam optimizer \citeA{kingma2014adam} with a base learning rate of $0.001$.
Below, we summarize conditions specfic to each dataset.

\paragraph{Mix-MNIST}
We trained for a maximum of $100$ epochs with a minibatch size of $128$, applying early stopping with ten epochs patience for the validation loss improvement.
We used no data augmentation for Mix-MNIST.

\paragraph{Mix-CIFAR-10}
We trained for a maximum of $250$ epochs with a minibatch size of $128$, 
applying a variant of warm-up and multi-step decay scheduling \citeA{he2016deep} as follows:
\begin{itemize}
\item A warm-up with five epochs
\item A multi-step decay that multiplies the learning rate by $0.1$ at the end of $100$ and $150$ epochs
\end{itemize}
We selected the best weights in terms of validation loss.
While training,
we applied data augmentation with random combinations of the following transformations:
\begin{itemize}
\item Rotation within $-180$ to $180$ degrees
\item Width and height shift within $\pm10$ pixels
\item Horizontal flip
\end{itemize}

\paragraph{MDS}
We trained for a maximum of $200$ epochs with a minibatch size of $128$, 
applying a warm-up and multi-step decay scheduling as follows:
\begin{itemize}
\item A warm-up with five epochs
\item A multi-step decay that multiplies the learning rate by $0.1$ at the end of $50$, $100$, and $150$ epochs
\end{itemize}
We recorded training weights for every five epochs,
and selected the best weights in terms of validation loss.
While training,
we applied data augmentation with random combinations of the following transformations:
\begin{itemize}
\item Rotation within $-20$ to $20$ degrees
\item Width and height shift within $\pm5$ pixels
\item Horizontal flip
\end{itemize}
For each image, the center $224 \times 224$ portion is cropped from the image after the data augmentation.

\subsection{Post-hoc calibrations and predictions}

We applied temperature scaling for CPE calibration and $\alpha$-calibration for obtaining CPE distributions.
For both calibration methods,
we used validation set, which was splited into $80\%$ calibration set for training and $20\%$ calibration-validation (cv) set for the validation of calibration.
We trained for a maximum of $50$ epochs using Adam optimizer with a learning rate of $0.001$,
applying early stopping with ten epochs patience for the cv loss improvement.
The loss functions of equation \eqref{eq:training_loss} and \eqref{eq:alpha_loss} were
used for CPE- and $\alpha$-calibration, respectively.
For the feature layer that used for $\alpha$-calibration,
we chose the penultimate layer that corresponded to the last dropout layer in this experiment.
The training scheme is the same as that of the CPE calibration,
except for a loss function that we described in \ref{aps:alpha_calib}.
We also used ensemble-based methods: Monte-Calro dropout (MCDO) \cite{gal2016dropout} and Test-time augmentation (TTA) \cite{ayhan2018test} for CPE distribution predictions,
which were both applicable to DNNs at prediction-time,
where $20$ MC-samples were used for ensemble.
A data augmentation applied in TTA was the same as that used in training,
and we only applied TTA for Mix-CIFAR-10 and MDS data.

\section{Additional experiments}

\label{aps:additional_experiments}

\subsection{Evaluations of class probability estimates}
\label{aps:results_cpes}

We present evaluation results of class probability estimates (CPEs) for Mix-MNIST and Mix-CIFAR-10
in Table \ref{tb:mix_mnist_cpe} and \ref{tb:mix_cifar10_cpe}, respectively.
Overall, CPE measures were comparable between the same datasets with different validation labels (two and five).
By comparing $\hLPS$ and $\hEL$, the relative ratio of $\hEL$ against irreducible loss could be evaluated
which was much higher in Mix-CIFAR-10 than in Mix-MNIST.
Among Raw predictions, temperature scaling kept accuracy and showed a consistent improvement in $\hEL$ and $\hCE$
as expected.
While TTA showed a superior performance over MCDO and Raw predictions in accuracy, $\hLPS$ and $\hEL$,
the effect of calibration methods for CPEs with ensemble-based predictions was not consistent,
which might be because calibration was not ensemble-aware.

\begin{table*}[!hb]  %
  \caption{Evaluations of CPEs for Mix-MNIST}
  \vspace{.2cm}
  \label{tb:mix_mnist_cpe}
  \centering
  \begin{tabular}{lcccccccc}
    \toprule
&
\multicolumn{4}{c}{Mix-MNIST(2)}
&
\multicolumn{4}{c}{Mix-MNIST(5)}
\\
\cmidrule(r){2-5}
\cmidrule(r){6-9}
Method
 & Acc $\uparrow$ 
 & $\hLPS$ $\downarrow$
 & $\hEL$ $\downarrow$
 & $\hCE$ $\downarrow$
 & Acc $\uparrow$ 
 & $\hLPS$ $\downarrow$
 & $\hEL$ $\downarrow$
 & $\hCE$ $\downarrow$
\\
\midrule
Raw & {\bf .9629} & .1386 & .0388 & .0518       & {\bf .9629} & .1386 & .0388 & .0518
\\
Raw+$\alpha$ & {\bf .9629} & .1386 & .0388 & .0518      & {\bf .9629} & .1386 & .0388 & .0518
\\
Raw+ts & {\bf .9629} & {\bf \underline{.1376}} & {\bf \underline{.0379}} & {\bf .0473}  & {\bf .9629} & {\bf \underline{.1376}} & {\bf \underline{.0379}} & {\bf .0475}
\\
Raw+ts+$\alpha$ & {\bf .9629} & {\bf \underline{.1376}} & {\bf \underline{.0379}} & {\bf .0473} & {\bf .9629} & {\bf \underline{.1376}} & {\bf \underline{.0379}} & {\bf .0475}
\\
\midrule
MCDO & {\bf \underline{.9635}} & .1392 & .0395 & {\bf \underline{.0425}}        & .9635 & .1392 & .0395 & {\bf \underline{.0425}}
\\
MCDO+$\alpha$ & .9621 & {\bf .1391} & {\bf .0394} & .0442       & {\bf \underline{.9644}} & {\bf .1387} & {\bf .0389} & .0463
\\
MCDO+ts & .9628 & .1408 & .0410 & .0632 & .9627 & .1402 & .0404 & .0638
\\
MCDO+ts+$\alpha$ & .9624 & .1412 & .0415 & .0653        & .9629 & .1407 & .0409 & .0631
\\
\bottomrule
\end{tabular}
\end{table*}
\begin{table*}[!hb]  %
  \caption{Evaluations of CPEs for Mix-CIFAR-10}
  \vspace{.2cm}
  \label{tb:mix_cifar10_cpe}
  \centering
  \begin{tabular}{lcccccccc}
    \toprule
&
\multicolumn{4}{c}{Mix-CIFAR-10(2)}
&
\multicolumn{4}{c}{Mix-CIFAR-10(5)}
\\
\cmidrule(r){2-5}
\cmidrule(r){6-9}
Method
 & Acc $\uparrow$ 
 & $\hLPS$ $\downarrow$
 & $\hEL$ $\downarrow$
 & $\hCE$ $\downarrow$
 & Acc $\uparrow$ 
 & $\hLPS$ $\downarrow$
 & $\hEL$ $\downarrow$
 & $\hCE$ $\downarrow$
\\
\midrule
Raw & {\bf .7965} & .3518 & .2504 & .1093       & {\bf .7965} & .3518 & .2504 & .1093
\\
Raw+$\alpha$ & {\bf .7965} & .3518 & .2504 & .1093      & {\bf .7965} & .3518 & .2504 & .1093
\\
Raw+ts & {\bf .7965} & {\bf .3437} & {\bf .2423} & {\bf .0687}  & {\bf .7965} & {\bf .3438} & {\bf .2423} & {\bf .0685}
\\
Raw+ts+$\alpha$ & {\bf .7965} & {\bf .3437} & {\bf .2423} & .0688       & {\bf .7965} & {\bf .3438} & {\bf .2423} & {\bf .0685}
\\
\midrule
MCDO & {\bf .7983} & .3488 & .2474 & .0955      & {\bf .7983} & .3488 & .2474 & .0955
\\
MCDO+$\alpha$ & .7968 & .3488 & .2474 & .0962   & .7965 & .3493 & .2479 & .0973
\\
MCDO+ts & .7972 & {\bf .3442} & {\bf .2428} & .0684     & .7977 & .3446 & .2431 & {\bf \underline{.0675}}
\\
MCDO+ts+$\alpha$ & .7965 & .3445 & .2430 & {\bf \underline{.0650}}      & .7975 & {\bf .3444} & {\bf .2430} & .0723
\\
\midrule
TTA & .8221 & .3230 & .2216 & {\bf .0877}       & .8221 & .3230 & .2216 & {\bf .0877}
\\
TTA+$\alpha$ & .8241 & {\bf \underline{.3221}} & {\bf \underline{.2206}} & .0884        & .8277 & {\bf \underline{.3223}} & {\bf \underline{.2209}} & .0894
\\
TTA+ts & .8257 & .3467 & .2452 & .1688  & {\bf \underline{.8279}} & .3466 & .2451 & .1707
\\
TTA+ts+$\alpha$ & {\bf \underline{.8277}} & .3446 & .2432 & .1697       & .8245 & .3460 & .2446 & .1726
\\
\bottomrule
\end{tabular}
\end{table*}

\subsection{Discussion on the effect of temperature scaling for disagreement probability estimates}
\label{aps:discussion_ts_for_dpe}

\begin{figure*}[!t]
  \centering
  \rule[-.0cm]{0cm}{0cm}
  \begin{tabular}{cc}
      \includegraphics[width=1.0\linewidth]{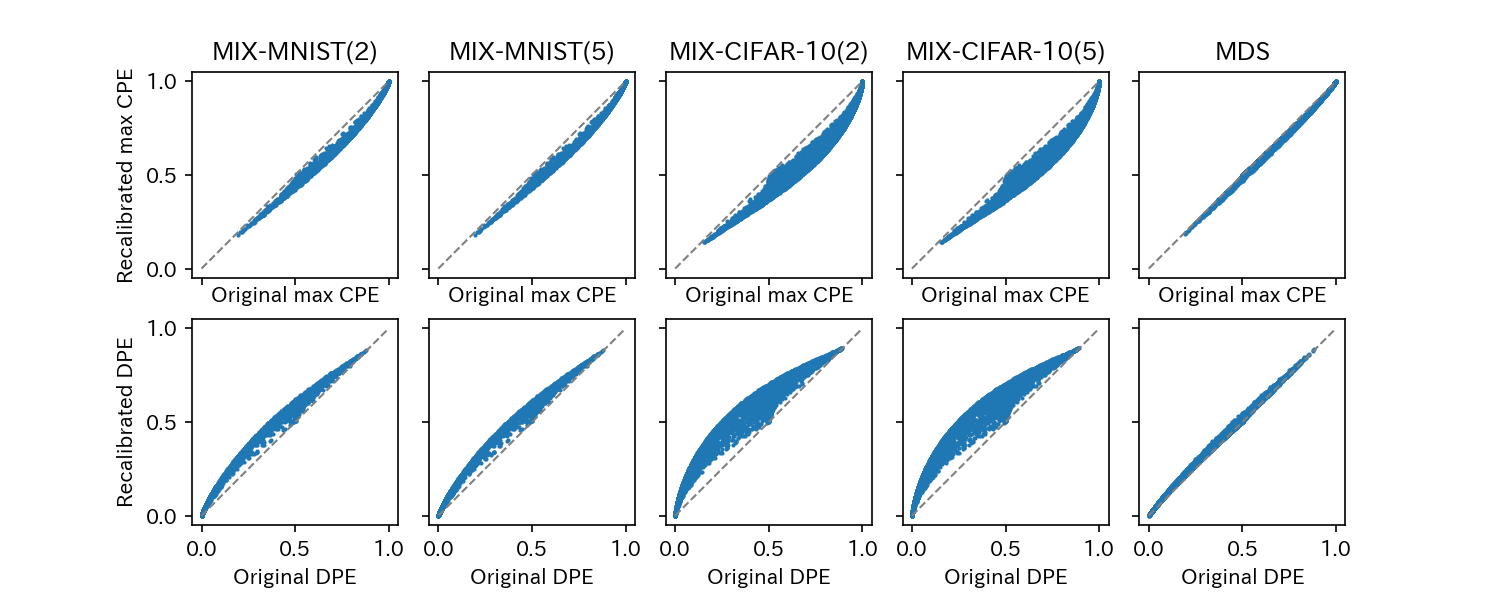}
      \\
      Changes in the maximum CPEs and DPEs with temperature scaling
  \end{tabular}
  \rule[-.0cm]{4cm}{0cm}
  \caption{
    Changes in the maximum class probability estimates (CPEs) and disagreement probability estimates (DPEs) with temperature scaling (ts) are presented for each instance.
    In contrast to the maximum CPEs, which are generally decreased with ts to compensate overconfidence,
    DPEs are increased with ts as discussed in Appendix \ref{aps:discussion_ts_for_dpe}.
    }
\label{fig:ts_cpe_dpe}
\end{figure*}

In Table \ref{tb:eval_DPEs},
it is observed that disagreement probability estimates (DPEs) are consistently degraded by temperature scaling (Raw+ts) from the original scores (Raw),
despite the positive effects for calibration of class probability estimates (CPEs) with ts.
Though we observe that the degradation can be overcome with $\alpha$-calibration (see Raw+ts+$\alpha$ or Raw+$\alpha$ in Table \ref{tb:eval_DPEs}), the mechanism that causes the phenomena is worth analyzing.
Since there exists well-known overconfidence in the maximum class probabilities from DNN classifiers \cite{guo2017calibration},
the recalibration of CPEs by ts tends to reduce the maximum class probabilities.
On the other hand, the amount of change in a DPE for each instance can be written as follows:
\begin{align}
\Delta^{\mathrm{D}} 
:= \varphi'^{\mathrm{D}} - \varphi^{\mathrm{D}}
= \sum_{k=1}^K (f_k + f_k') (f_k - f_k'),
\label{eq:delta_D}
\end{align}
where $f$ and $\varphi^{\mathrm{D}}$
denote the original CPEs and DPE, respectively,
and  $f'$ and $\varphi'^{\mathrm{D}}$
denote those after ts, respectively.
It is likely that $\Delta^{\mathrm{D}}$ takes a positive value
as the dominant term of $\Delta^{\mathrm{D}}$ in equation \eqref{eq:delta_D}
is $k$ with the maximum $f_k + f_k'$ value, 
where $f_k > f_k'$ is satisfied for overconfident predictions.
In fact, the averages of $\Delta^{\mathrm{D}}$
between Raw+ts and Raw for each of the five settings in Table \ref{tb:eval_DPEs} are all positive,
which are $0.027$, $0.025$, $0.101$, $0.103$, and $0.019$, respectively.
Instance-wise changes in the maximum CPEs and DPEs with ts are shown in Fig. \ref{fig:ts_cpe_dpe}.
Simultaneously,
DPEs without $\alpha$-calibration systematically overestimate the empirical disagreement probabilities, as shown in Fig. \ref{fig:mds_disagree1}.
Therefore, the positive $\Delta^{\mathrm{D}}$
means that $\varphi'^{\mathrm{D}}$
is even far from a target probability $\E[\phi^{\mathrm{D}} | X]$
than $\varphi^{\mathrm{D}}$ is, despite the improvement in CPEs with ts.

\subsection{Additional experiments for MDS data}
\label{aps:mds_full_results}

In addition to MDS data with single training labels per instance (MDS-1) used in the main experiment,
we trained and evaluated with full MDS data (MDS-full), where all the multiple training labels per example were employed.
Also,
we included additional CPE calibration methods: vector and matrix scaling (vs and ms, respectively),
which were introduced in Section \ref{aps:cpe_calib}, for these experiments.
We adopted an L2 regularization for vs and an ODIR for ms,
in which the following hyper-parameter candidates were examined:
\begin{itemize}
\item vs: $\lambda_b \in \{0.1, 1.0, 10\}$
\item ms: $(\lambda_b, \lambda_w) \in \{0.1, 1.0, 10\} \times \{0.1, 1.0, 10\}$
\end{itemize}
where $\lambda_b$ and $\lambda_w$ were defined in equation
\eqref{eq:calib_l2_reg} and \eqref{eq:calib_odir}, respectively.
The hyper-parameters were selected with respect to the best cv loss,
which were $\lambda_b = 0.1$ for vs and $\lambda_b = 1.0, \lambda_w = 10$ for ms in the single-training MDS data,
and $\lambda_b = 0.1$ for vs and $\lambda_b = 0.1, \lambda_w = 10$ for ms in the full MDS data.

\paragraph{Results}

We summarize the order-1 and -2 performance metrics
for predictions with MDS-1 and MDS-full datasets in Table \ref{tb:mds_results}.
For both datasets,
temperature scaling consistently improved (decreased) $\hEL$ and $\hCL$ for each of Raw, MCDO, and TTA predictions.
While vector scaling was slightly better at obtaining the highest accuracy than the other methods,
the effect of vs and ms for the metrics of probability predictions were limited.
As same as the results of synthetic experiments,
TTA showed a superior performance over MCDO and Raw predictions in accuracy, $\hLPS$ and $\hEL$.
Since $\hEL$ can be decomposed into $\hCL$ and the remaining term: $\hDL$ (Section \ref{aps:DL}),
the difference of predictors in CPE performance is clearly presented with 2d plots (Fig.\ref{fig:mds_cd_map} and \ref{fig:mds_cd_map_full}), which we call calibration-dispersion maps.
For order-2 metrics, both $\hL_{\phi^{\mathrm{D}}}$ and $\hCE_{\phi^{\mathrm{D}}}$ for DPEs
were substantially improved by $\alpha$-calibration,
espetially in $\hCE_{\phi^{\mathrm{D}}}$,
which was not attained with solely applying ensemble-based methods even with MDS-full.
This improvement of DPE calibration was also visually confirmed with reliability diagrams in Fig. \ref{fig:mds_disagree} and \ref{fig:mds_disagree_full}.
Though overall characteristics were similar between MDS-1 and MDS-full,
a substantial improvement in $\hLPS$, $\hEL$, and $\hL_{\phi^{\mathrm{D}}}$
was observed in MDS-full,
which seemed the results of enhanced probability predictions with additional training labels.

\begin{landscape}
\begin{table}[!tbp]
  \caption{Performance evaluations for MDS data}
  \label{tb:mds_results}
  \centering
  \begin{tabular}{lcccccccccccc}
    \toprule
&
\multicolumn{6}{c}{MDS-1}
&
\multicolumn{6}{c}{MDS-full}
\\
\cmidrule(r){2-7}
\cmidrule(r){8-13}
&
\multicolumn{4}{c}{Order-1 metrics}
&
\multicolumn{2}{c}{Order-2 metrics}
&
\multicolumn{4}{c}{Order-1 metrics}
&
\multicolumn{2}{c}{Order-2 metrics}
\\
\cmidrule(r){2-5}
\cmidrule(r){6-7}
\cmidrule(r){8-11}
\cmidrule(r){12-13}
Method
 & Acc $\uparrow$ 
 & $\hLPS$ $\downarrow$
 & $\hEL$ $\downarrow$
 & $\hCE$ $\downarrow$
 & $\hL_{\phi^D}$ $\downarrow$
 & $\hCE_{\phi^D}$ $\downarrow$
 & Acc $\uparrow$ 
 & $\hLPS$ $\downarrow$
 & $\hEL$ $\downarrow$
 & $\hCE$ $\downarrow$
 & $\hL_{\phi^D}$ $\downarrow$
 & $\hCE_{\phi^D}$ $\downarrow$
\\
\midrule
Raw & .9006 & .2515 & .0435 & .0600 & .1477 & .0628     & .8990 & .2460 & .0380 & .0590 & .1448 & .0539
\\
Raw+$\alpha$ & .9006 & .2515 & .0435 & .0600 & {\bf .1454} & {\bf .0406}        & .8990 & .2460 & .0380 & .0590 & {\bf .1430} & {\bf .0320}
\\
\cmidrule[.05mm]{1-13}
Raw+ts & .9006 & {\bf .2509} & {\bf .0430} & {\bf \underline{.0575}} & .1482 & .0663    & .8990 & {\bf .2459} & {\bf .0379} & {\bf .0587} & .1449 & .0545
\\
Raw+ts+$\alpha$ & .9006 & {\bf .2509} & {\bf .0430} & {\bf \underline{.0575}} & {\bf .1445} & {\bf .0261} & .8990 & {\bf .2459} & {\bf .0379} & {\bf .0587} & {\bf .1427} & {\bf .0269}
\\
\cmidrule[.05mm]{1-13}
Raw+vs & {\bf .9010} & .2515 & .0435 & .0579 & .1489 & .0696    & {\bf .8996} & .2461 & .0381 & .0602 & .1452 & .0563
\\
Raw+vs+$\alpha$ & {\bf .9010} & .2515 & .0435 & .0579 & {\bf .1449} & {\bf .0318}       & {\bf .8996} & .2461 & .0381 & .0602 & {\bf .1431} & {\bf .0310}
\\
\cmidrule[.05mm]{1-13}
Raw+ms & .8992 & .2532 & .0453 & .0656 & .1484 & .0663  & .8978 & .2480 & .0400 & .0710 & .1451 & .0563
\\
Raw+ms+$\alpha$ & .8992 & .2532 & .0453 & .0656 & {\bf .1453} & {\bf .0355}     & .8978 & .2480 & .0400 & .0710 & {\bf .1428} & {\bf .0274}
\\
\midrule
MCDO & .8983 & .2517 & .0437 & .0586 & .1470 & .0562    & .8989 & .2460 & .0380 & {\bf \underline{.0561}} & .1446 & .0505
\\
MCDO+$\alpha$ & .8996 & .2518 & .0438 & .0610 & {\bf .1450} & {\bf .0346}       & .9003 & {\bf .2458} & {\bf .0378} & .0568 & {\bf .1426} & {\bf .0237}
\\
\cmidrule[.05mm]{1-13}
MCDO+ts & .8986 & {\bf .2515} & {\bf .0435} & .0579 & .1479 & .0635     & .8995 & {\bf .2458} & {\bf .0378} & .0579 & .1447 & .0521
\\
MCDO+ts+$\alpha$ & .8991 & {\bf .2515} & {\bf .0435} & .0586 & {\bf .1442} & {\bf \underline{.0186}}    & .8997 & .2460 & .0380 & .0581 & {\bf .1423} & {\bf \underline{.0180}}
\\
\cmidrule[.05mm]{1-13}  
MCDO+vs & .8997 & .2521 & .0441 & .0589 & .1487 & .0664 & {\bf .9004} & .2461 & .0381 & .0593 & .1448 & .0531
\\
MCDO+vs+$\alpha$ & {\bf .9009} & .2519 & .0439 & {\bf \underline{.0575}} & {\bf .1446} & {\bf .0244}    & .9002 & .2460 & .0380 & .0579 & {\bf .1426} & {\bf .0228}
\\
\cmidrule[.05mm]{1-13}  
MCDO+ms & .8970 & .2536 & .0456 & .0678 & .1479 & .0612 & .8997 & .2477 & .0397 & .0704 & .1448 & .0529
\\
MCDO+ms+$\alpha$ & .8986 & .2534 & .0454 & .0667 & {\bf .1449} & {\bf .0279}    & .8986 & .2478 & .0399 & .0695 & {\bf .1424} & {\bf .0188}
\\
\midrule
TTA & .9013 & .2458 & .0378 & .0642 & .1441 & .0488     & .9069 & .2402 & .0322 & .0645 & .1425 & .0444
\\
TTA+$\alpha$ & .9025 & {\bf \underline{.2456}} & {\bf \underline{.0376}} & .0684 & {\bf .1428} & {\bf .0334}       & .9077 & .2402 & .0322 & .0650 & {\bf .1413} & {\bf .0277}
\\
\cmidrule[.05mm]{1-13}  
TTA+ts & .9012 & .2459 & .0379 & .0646 & .1448 & .0553  & .9074 & {\bf \underline{.2401}} & {\bf \underline{.0321}} & .0636 & .1425 & .0456
\\
TTA+ts+$\alpha$ & .9011 & .2458 & .0378 & {\bf .0632} & {\bf \underline{.1422}} & {\bf .0197}   & .9055 & .2403 & .0323 & {\bf .0633} & {\bf \underline{.1410}} & {\bf .0204}
\\
\cmidrule[.05mm]{1-13}
TTA+vs & {\bf \underline{.9031}} & .2471 & .0392 & .0671 & .1458 & .0598        & {\bf \underline{.9078}} & .2409 & .0329 & .0666 & .1431 & .0483
\\
TTA+vs+$\alpha$ & .9025 & .2470 & .0391 & .0657 & {\bf .1428} & {\bf .0247}     & .9068 & .2409 & .0329 & .0664 & {\bf .1416} & {\bf .0252}
\\
\cmidrule[.05mm]{1-13}
TTA+ms & .9001 & .2489 & .0410 & .0756 & .1456 & .0561  & .9034 & .2429 & .0349 & .0778 & .1431 & .0481
\\
TTA+ms+$\alpha$ & .8991 & .2487 & .0407 & .0750 & {\bf .1431} & {\bf .0268}     & .9030 & .2432 & .0352 & .0766 & {\bf .1414} & {\bf .0213}
\\
\bottomrule
\end{tabular}
\end{table}
\end{landscape}

\begin{figure*}[!hbt]
  \centering
  \rule[-.0cm]{0cm}{0cm}
  \begin{tabular}{cc}
      \includegraphics[width=0.42\linewidth]{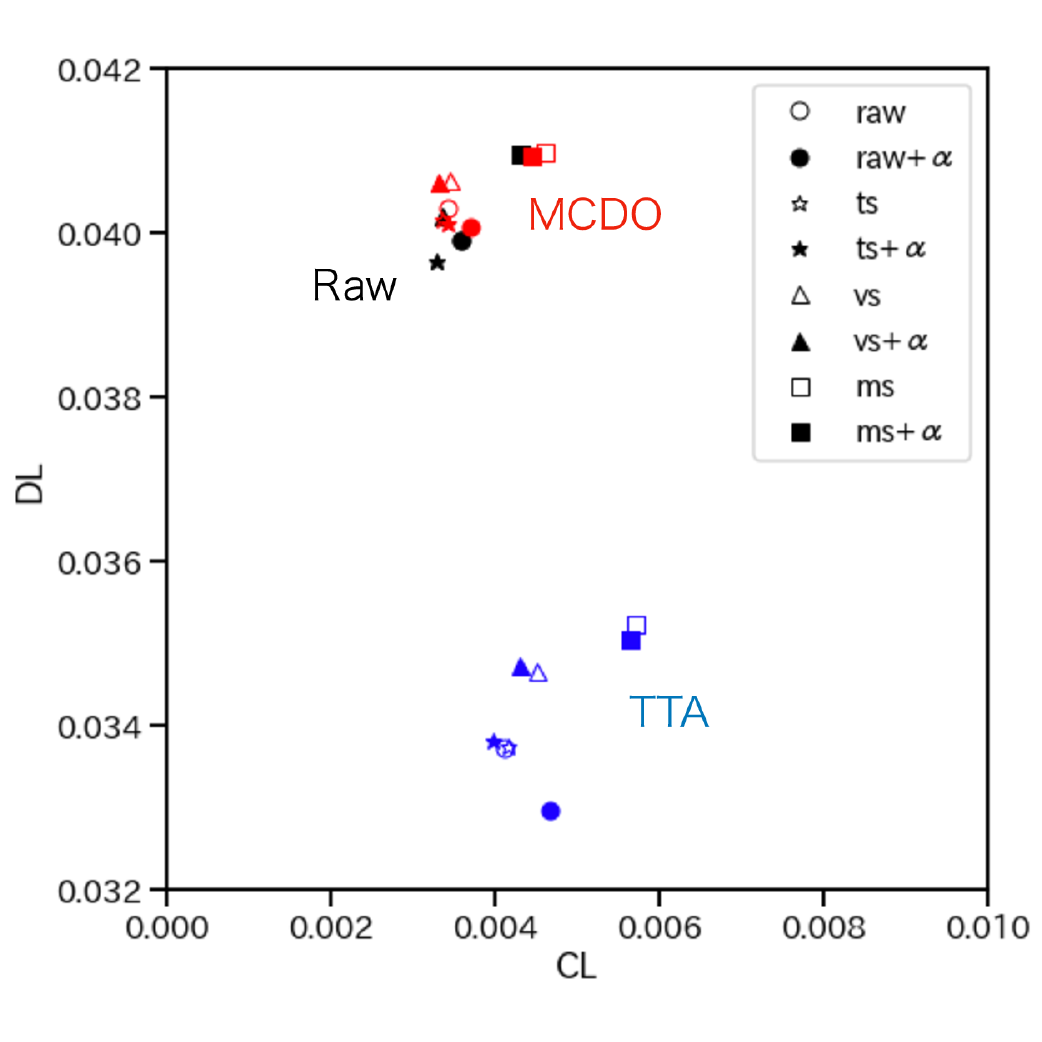}
      &
      \includegraphics[width=0.42\linewidth]{figs/mds_disagree_fig1.pdf}
      \\
      (a) Calibration-dispersion map
      &
      (b) Reliability diagram for disagreement probability
  \end{tabular}
  \rule[-.5cm]{4cm}{0cm}
  \caption{
(a) Calibration-dispersion map for the experiments with MDS-1 data.
(b) Reliability diagram of disagreement probability estimates (DPEs) for MDS-1 data,
where all the predictive methods in Table \ref{tb:mds_results} were compared. 
The dashed diagonal line corresponds to a calibrated prediction.
Calibration of DPEs was significantly enhanced with $\alpha$-calibration (solid lines) from the original ones (dotted lines).}
\label{fig:mds_cd_map}
\label{fig:mds_disagree}
\end{figure*}

\begin{figure}[!bp]
  \centering
  \rule[-.0cm]{0cm}{0cm}
  \begin{tabular}{cc}
    \includegraphics[width=0.42\linewidth]{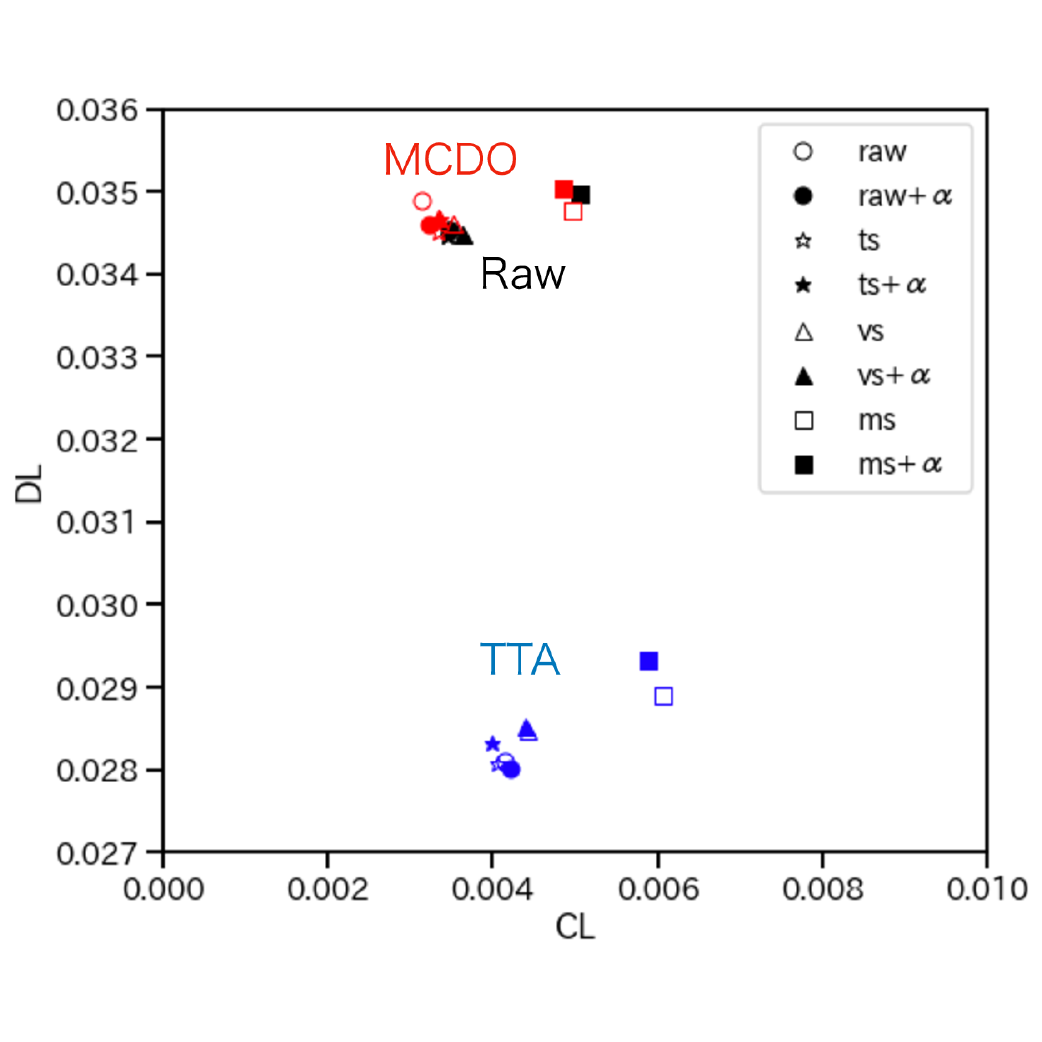}
    &
    \includegraphics[width=0.42\linewidth]{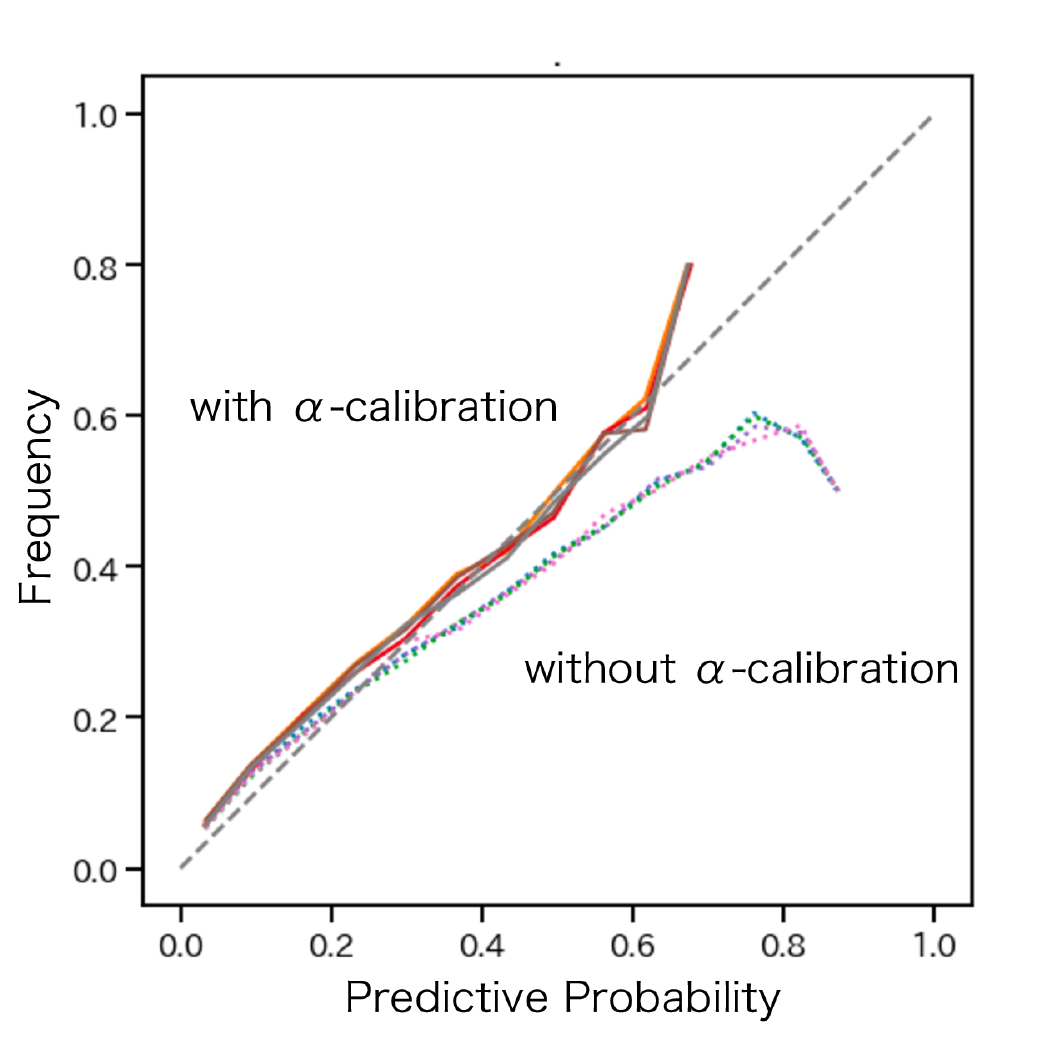}
  \\
  (a) Calibration-dispersion map
  &
  (b) Reliability diagram for disagreement probability
  \end{tabular}
  \rule[-.5cm]{4cm}{0cm}
  \caption{
(a) Calibration-dispersion map for the experiments with MDS-full data.
(b) Reliability diagram of disagreement probability estimates (DPEs) for MDS-full data,
where all the predictive methods in Table \ref{tb:mds_results} were compared. 
The dashed diagonal line corresponds to a calibrated prediction.
}
\label{fig:mds_cd_map_full}
\label{fig:mds_disagree_full}
\end{figure}

\end{document}